\documentclass{article}

\usepackage[preprint]{neurips_2025}

\usepackage[utf8]{inputenc} 
\usepackage[T1]{fontenc}    
\usepackage{hyperref}       
\usepackage{url}            
\usepackage{booktabs}       
\usepackage{amsfonts}       
\usepackage{nicefrac}       
\usepackage{microtype}      
\usepackage{xcolor}         

\usepackage{graphicx}
\usepackage{color,soul}
\usepackage{caption}
\usepackage{subcaption}
\usepackage{fancyhdr}
\usepackage{dsfont}
\usepackage{longtable}
\usepackage{makecell}
\usepackage{adjustbox} 
\usepackage{amsthm}
\usepackage{thmtools}
\usepackage{amsmath}
\usepackage{thmtools}
\usepackage{cleveref}
\usepackage{makecell}

\declaretheorem[name=Theorem,refname={Theorem,Theorems},Refname={Theorem,Theorems}]{theorem}

\newtheoremstyle{TheoremNum}
    {\topsep}{\topsep}              
    {\itshape}                      
    {}                              
    {\bfseries}                     
    {.}                             
    { }                             
    {\thmname{#1}\thmnote{ \bfseries #3}}
\theoremstyle{TheoremNum}
\newtheorem{thmn}{Theorem}

\renewcommand{\vec}[1]{\mathbf{#1}}
\renewcommand{\Re}{\mathbb{R}}

\newcommand{\hidden}{\vec{h}}
\newcommand{\backcast}{\vec{b}}
\newcommand{\proto}{\vec{p}}
\newcommand{\prediction}{\vec{f}}
\newcommand{\predictiongpd}{\widetilde{\prediction}}

\newcommand{\name}{{NIAQUE}}

\DeclareMathOperator{\linear}{\textsc{Linear}}

\DeclareMathOperator{\prototype}{\textsc{Prototype}}
\DeclareMathOperator{\fc}{\textsc{FC}}
\DeclareMathOperator{\relu}{\textsc{ReLU}}

\DeclareMathOperator{\smape}{s\textsc{mape}}

\DeclareMathOperator{\aad}{\textsc{aad}}
\DeclareMathOperator{\bias}{\textsc{bias}}
\DeclareMathOperator{\crps}{\textsc{CRPS}}

\DeclareMathOperator{\rmse}{\textsc{rmse}}
\DeclareMathOperator{\rmsle}{\textsc{rmsle}}

\DeclareMathOperator{\coverage}{\textsc{coverage}}

\DeclareMathOperator{\ci}{\textsc{CI}}

\DeclareMathOperator{\sgn}{sgn}
\newcommand{\nbeatsinput}{\vec{x}}

\title{Probabilistic Pretraining for Neural Regression}

\author{%
  Boris N. Oreshkin \\
  SCOT Forecasting\\
  Amazon\\
  \texttt{oreshkin@amazon.com} \\
  \And
  Shiv Tavker \\
  Pricing and Promotions\\
  Amazon\\
  \texttt{tavker@amazon.com} \\
  \And
  Dmitry Efimov \\
  SCOT Forecasting\\
  Amazon\\
  \texttt{defimov@amazon.com} \\
}

\begin{document}

\maketitle

\begin{abstract}
Transfer learning for probabilistic regression remains underexplored. This work closes this gap by introducing NIAQUE, Neural Interpretable Any-Quantile Estimation, a new model designed for transfer learning in probabilistic regression through permutation invariance. We demonstrate that pre-training NIAQUE directly on diverse downstream regression datasets and fine-tuning it on a specific target dataset enhances performance on individual regression tasks, showcasing the positive impact of probabilistic transfer learning. Furthermore, we highlight the effectiveness of NIAQUE in Kaggle competitions against strong baselines involving tree-based models and recent neural foundation models TabPFN and TabDPT. The findings highlight NIAQUE's efficacy as a robust and scalable framework for probabilistic regression, leveraging transfer learning to enhance predictive performance.
\end{abstract}

\section{Introduction}
Tabular data is a cornerstone of real-world applications, spanning diverse domains such as healthcare~\citep{Rajkomar2018}, where electronic health records enable disease prediction and treatment optimization; real estate~\cite{DeCock2011}, where features like location and property characteristics drive house price prediction; energy forecasting~\cite{Olson2017}, where meteorological and historical data drive wind power generation estimates; and e-commerce \cite{McAuley2015}, where dynamic market data enables accurate product price prediction and optimization. Traditional tree-based models like Random Forests~\cite{breiman2001random}, XGBoost~\cite{Chen_2016}, LightGBM~\cite{NIPS2017_6449f44a}, and CatBoost~\cite{prokhorenkova2019catboost} have dominated tabular predictive modeling. These models are prized for their simplicity, interpretability, and strong performance across a range of structured data problems. While traditional methods remain effective, recent advancements in deep learning have introduced novel approaches for tabular data modeling.
Architectures such as TabNet~\cite{arik2021tabnet} and TabTransformer~\cite{Huang2021TabTransformer} have narrowed the performance gap with classical models while enabling end-to-end training and multimodal data integration.
At the same time, Transfer learning --- a paradigm where models are pretrained on large-scale datasets to improve downstream performance --- has been transformative in computer vision and natural language processing. However, its application to probabilistic regression for tabular data remains largely underexplored, due to significant focus on classification tasks~\citep{hollmann2023tabpfn, levin2022transfer}. In this work, we identify and bridge several key research gaps. First, we propose NIAQUE, a novel probabilistic regression model. We show, for the first time, that a \emph{probabilistic} regression model can be co-trained on multiple disjoint datasets, exhibiting positive transfer and excellent scalability when fine-tuned on unseen new regression datasets. NIAQUE compares favorably against strong tree-based baselines, Transformer baseline and existing models such as TabDPT and TabPFN, despite being trained solely on a collection of downstream probabilistic regression tasks, with no bells and whistles. Additionally, its probabilistic nature enables interpretable feature importance analysis via marginal posterior distributions, facilitating the identification of key predictive factors, enhancing model transparency, trust and reliability. Since existing multi-dataset tabular benchmarks are predominantly focused on classification problems, to support our study, we introduce a new multi-dataset regression benchmark and train multiple baseline models across all its datasets in a multi-task fashion. This benchmark comprises 101 diverse datasets from various domains, with varying sample sizes and feature dimensions. Furthermore, in a real-world case study involving Kaggle regression challenges, NIAQUE leverages cross-dataset pretraining to achieve competitive performance against highly feature engineered hand-crafted solutions and wins over TabDPT and TabPFN baselines. Our contributions can be summarized as follows.

\begin{enumerate}
\item We introduce NIAQUE, a deep probabilistic regression model, trained across diverse tabular datasets.
\item We establish a theoretical framework showing that NIAQUE approximates the inverse of the posterior distribution.
\item We empirically validate NIAQUE’s superiority over strong baselines, including CatBoost, XGBoost, LightGBM, Transformer, TabDPT and TabPFN in transfer learning setting.
\end{enumerate}

\subsection{Related Work}

\textbf{Probabilistic Regression}
This work builds on probabilistic time-series modeling approach~\cite{smyl2024anyquantile}, refining its theoretical underpinnings and extending architectural design for tabular transfer learning applications. Alternative methods, such as Neural Processes~\cite{garnelo2018neuralprocesses} and Conditional Neural Processes~\cite{garnelo2018conditional}, offer conditional probabilistic solutions to regression but are constrained to fixed-dimensional input spaces, limiting their applicability to cross-dataset, multi-task regression. Our approach effectively transfers knowledge across datasets with varying feature spaces and target domains, establishing a flexible and scalable framework for conditional probabilistic regression.

\textbf{Multi-task  and Transfer Learning} have been dominant paradigms in computer vision~\cite{sun2021monocular,radford2021learning} and language modeling~\cite{devlin2019bert}, achieving significant breakthroughs by leveraging shared representations across tasks and datasets. More recently, transfer learning has gained traction in univariate time-series forecasting~\cite{garza2023timegpt1,ansari2024chronos}, enabling improved generalization across datasets. However, in the domain of tabular data processing, probabilistic transfer learning remains underexplored, with a primary focus on training dataset-specific models for classification problems and point regression tasks. NIAQUE closes the gap by enabling probabilistic transfer across diverse tabular datasets.

\textbf{Deep Learning vs. Tree-based Models.}
Prior work has extensively benchmarked deep learning models against tree-based approaches 
for tabular data, with emphasis on classification tasks. For instance, Transformers have been evaluated across 20 classification datasets, respectively, in~\cite{muller2022transformers}, while MLPs were compared against TabNet and tree-based models on 40 classification datasets in~\cite{kadra2021well}. Similarly, \cite{grinsztajn2022why} benchmarked architectures like Transformers, ResNet, and MLPs against tree-based models across 45 datasets, where only about half were regression problems. Importantly, these models were trained independently for each dataset, limiting their applicability to transfer learning. Closest to our work, TabPFN~\citep{hollmann2023tabpfn} and TabDPT~\citep{ma2024tabdpt} advocate Transformer-based approaches to tabular tasks and consider transfer learning through retrieval. Both of them share architectural approach, with~\cite{ma2024tabdpt} putting more emphasis on real data pretraining while~\cite{hollmann2023tabpfn} focusing on synthetic data pretraining. Unlike these works, we (i) explore a different architectural approach based on deep prototype aggregation; (ii) focus on probabilistic pretraining and transfer learning, a new under-explored problem area; (iii) show that compared to TabPFN and TabDPT, our architecture probabilistically pretrained directly on a large collection of downstream regression tasks, results in better empirical accuracy in the wild on realistic Kaggle competitions.

\textbf{Permutation-invariant Representation Learning.}
In terms of architectural approach, our work builds upon advancements in permutation-invariant representations, enabling multi-task learning across datasets with variable feature spaces. \cite{oreshkin2022protores} proposed a related architecture for human pose completion in animation, which we extend for any-quantile modeling in tabular regression. Other architectures, such as PointNet~\citep{qi2017pointnet} and DeepSets~\citep{zaheer2017deepsets}, use pooling techniques to handle variable input dimensions in 3D point clouds and concept retrieval, respectively, and are further generalized by ResPointNet~\cite{niemeyer2019occupancy}. Similarly, Prototypical Networks~\citep{snell2017prototypical} leverage average-pooled embeddings for few-shot classification, and Transformer-based architectures~\citep{vaswani2017attention} have successfully demonstrated their adaptability in natural language processing tasks with variable size inputs.

\subsection{Preliminaries and Background}\label{sec:prelim}

\textbf{Notations:} Let \( \mathbb{R} \) denote the set of real numbers and $\mathcal{U}(0,1)$ the uniform distribution over the interval $(0,1)$. For a vector \( \vec{x} \), we denote its dimensionality as \( |\vec{x}| \). For a random variable \( Y \) with cumulative distribution function (CDF) \( F(y) = P(Y \leq y) \), the \( q \)-th quantile \( q \in (0,1) \) is defined as:  
\[
F^{-1}(q) = \inf \{ y \in \mathbb{R} : F(y) \geq q \}.
\]

\noindent\textbf{Problem Formulation:} Let $\mathcal{X}$ be the input feature space and $\mathcal{Y} \subseteq \mathbb{R}$ be the space of the target variable. We consider a probability distribution $\mathcal{D}$ over $\mathcal{X} \times \mathcal{Y}$. For any instance $\vec{x} \in \mathcal{X}$, the relationship between features and target variable is given by:
\begin{align}
    y = \Psi(\vec{x}, \varepsilon)
\end{align}
where $\Psi: \mathcal{X} \times \mathcal{E} \rightarrow \mathcal{Y}$ is an unknown non-linear function and $\varepsilon \in \mathcal{E}$ represents stochastic noise with unknown distribution. 

Given a finite training sample $S = \{(\vec{x}_i, y_i)\}_{i=1}^N$ drawn i.i.d. from $\mathcal{D}$, we aim to learn a probabilistic regression function $f_{\theta}: \mathbb{R}^{|\vec{x}| \times Q} \rightarrow \mathbb{R}^Q$, parameterized by $\theta \in \Theta$, which maps an input $\vec{x}$ to a $Q$-tuple of quantiles $(q_1, ..., q_Q)$, where $q_i \in (0,1)$ for $i \in [Q]$, thereby capturing the conditional distribution of $y|\vec{x}$.

\noindent\textbf{Performance Metrics}
Let $y_i$ denote the ground truth sample and $\hat{y}_{i,q}$ its $q$-th quantile prediction for a dataset with $S$ samples. To evaluate the quality of distributional predictions, we use Continuous Ranked Probability Score (CRPS). The theoretical definition of CRPS for a predicted cumulative distribution function $F$ and observation $y$ is:
\begin{equation} \label{eqn:crps_theoretical}
\text{CRPS}(F,y) = \int_{\mathbb{R}}
       \left(F(z)- \mathds{1}_{\{z \geq y\}} \right)^{2}\text{d}z,
\end{equation}
where $F: \mathbb{R} \rightarrow [0,1]$ is the predicted CDF derived from the quantile predictions,
and $\mathds{1}_{\{z \geq y\}}$ is the indicator function. For practical computation with finite samples $S$ and a discrete set of $Q$ quantiles, we approximate this using:
\begin{align}\label{eqn:crps_sample}
\crps = \frac{2}{S Q} \sum_{i=1}^S \sum_{j=1}^Q \rho(y_i, \hat{y}_{i, q_j})
\end{align}
where $\rho(y, \hat{y}_q)$ is the quantile loss function defined as:
\begin{align}\label{eqn:quantile-loss}
\rho(y, \hat{y}_q) = (y - \hat{y}_q)(q - \mathds{1}_{\{y \leq \hat{y}_q\}})
\end{align}
Additional performance metrics used in this work are defined in Appendix~\ref{sec:performance_metrics}.

\section{NIAQUE and Transfer Learning}
In this section, we present NIAQUE (Neural Interpretable Any-Quantile Estimation), a probabilistic regression model. We first introduce the any-quantile learning approach as a general solution to the probabilistic regression problem defined in Section~\ref{sec:prelim}. We prove that this approach converges to the inverse cumulative distribution function of the conditional distribution, providing a theoretical foundation for our method. We then detail NIAQUE's neural architecture, demonstrate how it enables transfer learning across diverse tabular datasets, and present an approach to model interpretability based on probabilistic considerations.

\subsection{Any-Quantile Learning}

We formulate the any-quantile learning approach by augmenting the input space to include a quantile level $q \in (0,1)$, allowing the neural network $f_{\theta}$ to learn mappings from $(\vec{x}, q)$ to the corresponding $q$-th conditional quantile of the target variable $y|\vec{x}$. Let $\widehat{y}_q = f_{\theta}(\vec{x}, q)$ represent the predicted $q$-th quantile of the conditional distribution of $y|\vec{x}$. The objective is to learn parameters $\theta$ that minimize the expected quantile loss:
\begin{align}
    \min_{\theta} \mathbb{E}_{(\vec{x},y) \sim \mathcal{D}, q \sim \mathcal{U}(0,1)}[\rho(y, f_{\theta}(\vec{x}, q))] \,,
\end{align}
where $\rho(\cdot, \cdot)$ is the quantile loss function defined in Equation~\eqref{eqn:quantile-loss}.

We use gradient descent and mini-batch to learn the parameters. Precisely, the neural network is trained on dataset of $S$ samples, $(\vec{x}_i, y_i)$ drawn from the joint probability distribution $\mathcal{D}$. During training the quantile value $q$ is sampled from $\mathcal{U}(0, 1)$ and the loss is minimized using stochastic gradient descent (SGD). For a mini-batch of size $B$, the parameter update at iteration $k$ is:
\begin{equation}
\theta_{k+1} = \theta_{k} - \eta_k \nabla_{\theta} \frac{1}{B}\sum_{i=1}^B \rho(y_i, f_{\theta}(\vec{x}_i, q_i)) \,.
\label{eqn:mini_batch_update}
\end{equation}
As $k \to \infty$, the parameters converge to the solution of the following empirical risk minimization problem~\citep{Karimi2016linear}:
\begin{equation}
\theta^{*} = \arg\min_{\theta \in \Theta} \frac{1}{S} \sum_{i=1}^S \rho(y_i, f_{\theta}(\vec{x}_i, q_i)) \,.
\label{eqn:theta_optimal}
\end{equation}
By the strong law of large numbers, as $S$ grows, the empirical risk converges to the expected quantile loss:
\begin{equation}
\mathbb{E}_{\vec{x}, y} \mathbb{E}_{q} \rho(y, f_{\theta}(\vec{x}, q)) = \mathbb{E}_{\vec{x}, y} \int_{0}^{1} \rho(y, f_{\theta}(\vec{x}, q)) dq \,.
\label{eqn:sum_to_integral}
\end{equation}
This expected loss has a direct connection to the Continuous Ranked Probability Score (CRPS), which can be expressed as an integral over quantile loss~\citep{tilmann2011comparing}:
\begin{equation}
\crps(F, y) = 2 \int_{0}^{1} \rho(y, F^{-1}(q)) dq \,.
\label{eqn:crps_expected}
\end{equation}
Based on this fact, the following theorem proves that the expected pinball loss~\eqref{eqn:sum_to_integral} is minimized when $f_{\theta}(\vec{x}, q)$ corresponds to the inverse of the posterior CDF $P_{y|\vec{x}}$.
\begin{theorem} \label{thm:crps_minimizer} 
Let $F$ be a probability measure over variable $y$ such that inverse $F^{-1}$ exists and let $P_{y,\vec{x}}$ be the joint probability measure of variables $\vec{x}, y$. Then the expected loss, $\mathbb{E}_{\vec{x},y,q}\, \rho(y, F^{-1}(q))$, is minimized if and only if $F = P_{y|\vec{x}}$.
\end{theorem}
The following conclusions emerge. First, the quantile loss SGD update~\eqref{eqn:mini_batch_update} optimizes the empirical risk~\eqref{eqn:theta_optimal} corresponding to the expected loss~\eqref{eqn:sum_to_integral}. Based on~(\ref{eqn:sum_to_integral},\ref{eqn:crps_expected}) and Theorem~\ref{thm:crps_minimizer}, $f_{\theta^{\star}} = \arg\min_{f_{\theta}} \mathbb{E}_{\vec{x},y,q}\, \rho(y, f_{\theta}(\vec{x}, q))$ has a clear interpretation as the inverse CDF corresponding to $P_{y|\vec{x}}$. Second, as both the SGD iteration index $k$ and training sample size $S$ increase, and if $f_{\theta}$ is implemented as an MLP whose width and depth scale appropriately with sample size $S$, then~\cite[Theorem 1]{farrell2021Deep} implies that the SGD solution converges to $f_{\theta^{\star}}(\vec{x}, q) \equiv P_{y|\vec{x}}^{-1}(q)$. Therefore, given uniform $q \sim \mathcal{U}(0,1)$, $\widehat{y}_q = f_{\theta^{\star}}(\vec{x}, q)$ has the interpretation of a sample from the posterior distribution $p(y | \vec{x})$, which follows from the proof of the inversion method~\cite[Theorem 2.1]{devroy86nonuniform}.

\begin{figure}[t]
    \centering
    \hspace*{-0.5cm}
    \includegraphics[width=0.6\linewidth]{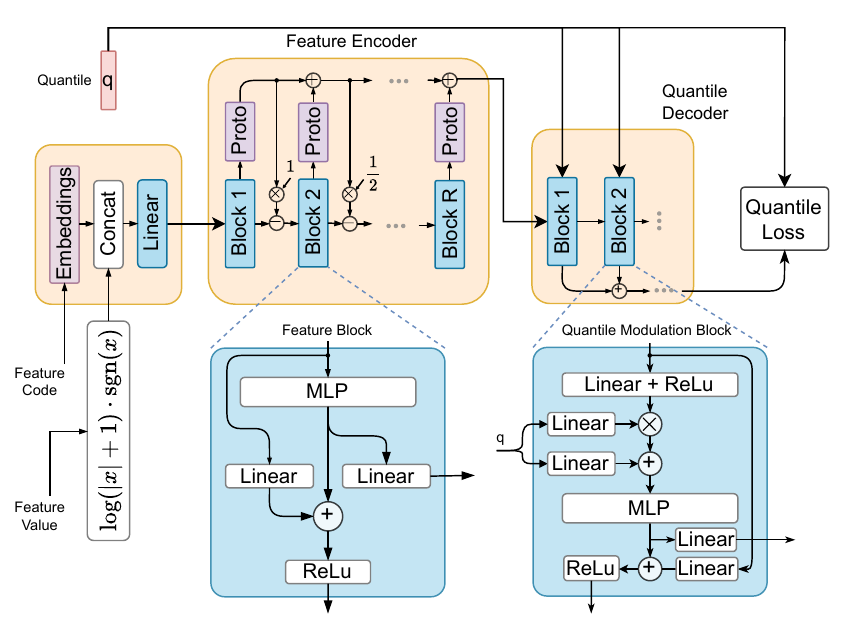}
    \caption{NIAQUE's encoder-decoder architecture transforms variable-dimensional inputs into fixed-size representations, enabling transfer learning and multi-task knowledge sharing across datasets.}
    \label{fig:niaque_architecture}
\end{figure}

\subsection{Neural Encoder-Decoder Architecture}

NIAQUE adopts a modular encoder-decoder design (Fig.~\ref{fig:niaque_architecture}) to process observation samples $\vec{x}_i$ with variable dimensionality $d_i$. The encoder maps each observation into a fixed-size latent embedding of dimension $E$, enabling downstream processing independent of input dimensionality. Feature values and associated codes (IDs) (dimension $1 \times d_i$) are embedded into a tensor of size $1 \times d_i \times E_{in}$, where $E_{in}$ is the embedding size per feature. 
These embeddings are aggregated using a prototype-based method to generate the latent observation representation. The decoder conditions this representation on arbitrary-length quantile vectors $\vec{q} \in \Re^Q$, modulating the output using FiLM-based transformations~\cite{perez2017film}. This separates input processing and quantile conditioning, achieving computational efficiency of $O(d_i + Q)$ per sample $\vec{x}_i$, compared to $O(d_i Q)$ complexity required to process quantiles and observations jointly. Note that our processing is linear if feature dimensionality, as opposed to quadratic scaling of attention-based approaches.
\\

\noindent\textbf{Inputs:} NIAQUE incorporates both raw value and learnable embedding of its ID (encoded as an integer) for each feature in the observation vector $\vec{x}$. The feature ID embedding learns feature-specific statistical properties, inter-feature dependencies, and their relationship with the target. The embedded feature ID is concatenated with its value after the log-transform:
\begin{align} \label{eqn:log_transformation}
    z = \log(|x| + 1) \cdot \sgn(x) \,,
\end{align}
which normalizes the features' dynamic range, aligning it with that of ID embeddings while preserving sign information, facilitating stable training, as validated by the ablation study in Appendix~\ref{sec:niaque-ablation-appendix}.

\noindent\textbf{Feature Encoder:} The encoder employs a two-loop residual network architecture to efficiently handle variable-dimensional inputs. The following equations define the encoder’s transformations, with the sample index $i$ 
omitted for brevity. Let the encoder input be $\vec{x}_{in} \in \Re^{d \times E_{in}}$, where $d$ is the number of features and $E_{in}$ is the embedding size per feature. A fully-connected layer $\fc_{r,\ell}$ in residual block $r \in \{1, \ldots, R\}$, layer $\ell \in \{1, \ldots, L\}$, with weights $\vec{W}_{r,\ell}$ and biases $\vec{a}_{r,\ell}$, is defined as:
$$\fc_{r,\ell}(\hidden_{r, \ell-1}) \equiv \relu(\vec{W}_{r,\ell} \hidden_{r,\ell-1} + \vec{a}_{r,\ell})\,.$$ We also define a prototype layer as: $\prototype(\vec{x}) \equiv \frac{1}{d} \sum_{i=1}^d \vec{x}[i,:]\,.$ The observation encoder is then described by the following equations:
\begin{align}
    \vec{x}_{r} &= \relu(\backcast_{r-1} - 1/(r-1) \cdot \proto_{r-1}), \label{eqn:prorez_encoder_in} \\
    \hidden_{r,1} &= \fc_{r,1}(\vec{x}_{r}), \  \ldots, \  \hidden_{r,L} = \fc_{r,L}(\hidden_{r,L-1}),  \label{eqn:prorez_encoder_mlp} \\
    \backcast_{r} &= \relu(\vec{L}_{r} \vec{x}_{r} + \hidden_{r,L}), \    \prediction_{r} = \vec{F}_{r} \hidden_{r,L}, \label{eqn:prorez_encoder_residual_1} \\
    \proto_{r} &= \proto_{r-1} + \prototype(\prediction_{r}). \label{eqn:prorez_encoder_proto}
\end{align}
These equations implement a dual-residual mechanism:

1. Equations \eqref{eqn:prorez_encoder_mlp} and \eqref{eqn:prorez_encoder_residual_1} form an MLP with a residual connection (see Feature Block in Fig.~\ref{fig:niaque_architecture} bottom left).

2. Equations~\eqref{eqn:prorez_encoder_in} and~\eqref{eqn:prorez_encoder_proto} form a second residual loop with the following key properties: a) Eq.~\eqref{eqn:prorez_encoder_proto} consolidates individual feature encodings into a prototype-based representation of the observation. b) Eq.~\eqref{eqn:prorez_encoder_in} introduces an inductive bias by enforcing a delta-mode constraint, ensuring that feature contributions are only relevant when they deviate from the existing observation embedding, $\proto_{r-1}$. c) The observation representation accumulates across residual blocks Eq.~\eqref{eqn:prorez_encoder_proto}, effectively implementing skip connections.

\noindent\textbf{Quantile Decoder:} The decoder implements a fully-connected conditioned residual architecture (Fig.~\ref{fig:niaque_architecture}, top-right) with FiLM-modulated MLP blocks (Fig.~\ref{fig:niaque_architecture}, bottom-right). Taking the observation embedding $\widetilde{\backcast}_{0}=\proto_{R} \in \Re^{E}$ as input, it generates quantile-modulated representations $\widetilde{\vec{f}}_{R} \in \Re^{Q\times E}$ for quantiles $\vec{q} \in \Re^Q$ through:

\begin{align}  \label{eqn:forward_kinematics_decoder}
\begin{split}
    \hidden_{r,1} &= \fc_{r,1}^{\textrm{QD}}(\widetilde{\backcast}_{r-1})\,, \quad \gamma_{r}, \beta_{r} = \linear_{r}(\vec{q}) \,, \\
    \hidden_{r,2} &= \fc_{r,1}^{\textrm{QD}}((1 + \gamma_{r})\cdot\hidden_{r,1} + \beta_{r})\,, \\
    &\ldots \\
    \hidden_{r,L} &= \fc_{r,L}^{\textrm{QD}}(\hidden_{r,L-1})\,,  \\
    \widetilde{\backcast}_{r} &= \relu(\vec{L}_{r}^{\textrm{QD}} \widetilde{\backcast}_{r-1} + \hidden_{r,L}), \ \    \predictiongpd_{r} = \predictiongpd_{r-1} + \vec{F}_{r}^{\textrm{QD}} \hidden_{r,L}\,.
\end{split}
\end{align}
The final prediction $\widehat{\vec{y}}_{q} \in \Re^Q$ is obtained via linear projection: $\widehat{\vec{y}}_{q} = \linear[\predictiongpd_{r}]$.

\subsection{Interpretability}\label{ssec:interpretability_theory}

NIAQUE's probabilistic framework facilitates interpretability via quantile predictions conditioned on individual features. Given $f_{\theta}(\vec{x}_{s}, q)$ as NIAQUE's estimate of quantile $q$ using only feature $\vec{x}_{s}$, the posterior confidence interval for this feature is defined as:
\begin{align}
    \ci_{\alpha,s} = f_{\theta}(\vec{x}_{s}, 1-\alpha/2) - f_{\theta}(\vec{x}_{s}, \alpha/2) \,,
\end{align}
where $1-\alpha$ represents the probability that the target true value lies within the interval. Intuitively, more informative features produce narrower confidence intervals. We leverage this to quantify feature importance through normalized weights:
\begin{equation}  \label{eqn:niaque_weights}
W_{s} = \frac{\overline{W}_{s}}{\sum_s \overline{W}_{s}}\,, \ \ \overline{W}_{s} = \frac{1}{\overline{\ci}_{0.95,s}}\,, \ \ 
\overline{\ci}_{\alpha,s} = \frac{1}{S}\sum_{i} f_{\theta}(\vec{x}_{s,i}, 1-\alpha/2) - f_{\theta}(\vec{x}_{s,i}, \alpha/2) \,.
\end{equation}
where $\overline{\ci}_{\alpha,s}$ is the average confidence interval width over validation samples $\{\vec{x}_{i} : \vec{x}_{i} \in \mathcal{D}_{\text{val}}\}$. To enhance marginal distribution modeling and support interpretability, we introduce single-feature samples during training, comprising approximately 5\% of the dataset. An ablation study in Appendix~\ref{sec:niaque-ablation-appendix} confirms the necessity of this augmentation for robust feature importance estimation.

\subsection{Transfer Learning}

NIAQUE facilitates effective transfer learning through two core mechanisms ensuring effective knowledge transfer across diverse tabular regression tasks with varying feature spaces. First, its feature ID embeddings learn both dataset-specific and cross-dataset relationships, as evidenced by the structured representation space observed in Fig.~\ref{fig:niaque_interpretability_and_clusters}.
Second, its prototype-based aggregation enables flexible processing of arbitrary feature combinations, inherently supporting cross-dataset learning. 

The transfer learning process consists of two phases. During pretraining, NIAQUE learns from multiple heterogeneous datasets simultaneously, sampling rows uniformly at random from all datasets, with each dataset contributing only its relevant features. The model processes these diverse inputs through shared parameters, learning both task-specific characteristics and generalizable patterns. The learned embeddings capture statistical properties at multiple levels: individual feature distributions, feature interactions within datasets, and common patterns across different regression problems. This pretrained knowledge can then be leveraged in two ways: (1) Zero-shot transfer to seen datasets, enabling immediate application without additional training. The model learns to distinguish and process dataset-specific features through their semantic embeddings, whereas shared model parameters enable knowledge transfer across related datasets via multi-task training. (2) Few-shot transfer to unseen datasets through fine-tuning, where the pretrained feature representations provide a strong foundation for learning new tasks with minimal data via strong regression prior stored in pretrained model weights.

The effectiveness of this approach stems from NIAQUE's ability to maintain dataset-specific information while learning transferable representations. Learnable feature ID embeddings act as task identifiers, allowing the model to adapt its processing based on the combination of input features, effectively serving as an implicit task ID.  This capability is particularly valuable in tabular domains, where feature relationships and their predictive power can vary significantly across tasks. Our experimental results validate both the efficacy of cross-dataset pretraining and the model’s adaptability to novel regression tasks via fine-tuning, establishing NIAQUE as a robust framework for transfer learning in tabular domains.

\section{Empirical Results}
We conduct extensive experiments to evaluate NIAQUE's effectiveness for transfer learning in tabular regression. Our evaluation addresses three key aspects: 
1) the model's transfer learning capabilities across both seen and unseen tasks, 
2) the quality of learned representations and interpretability, and 
3) its practical effectiveness in a real-world competition setting.


\subsection{Datasets and Experimental Setup}

\textbf{Dataset and Evaluation:} We introduce TabRegSet-101 (Tabular Regression Set 101), a curated collection of 101 publicly available regression datasets gathered from UCI~\cite{kelly2017uci}, Kaggle~\cite{kaggle2024kaggle}, PMLB~\cite{romano2021pmlb,olson2017pmlb}, OpenML~\cite{vanschoren2013openml}, and KEEL~\cite{fdez2011keel}. These datasets span diverse domains, including \textit{Housing and Real Estate}, \textit{Energy and Efficiency}, \textit{Retail and Sales}, \textit{Computer Systems}, \textit{Physics Models} and \textit{Medicine} and exhibit different characteristics in terms of sample size and feature dimensionality. The datasets, along with their sample count, number of variables and source information are listed in Appendix~\ref{sec:lprm-101}. To ensure balanced evaluation across datasets of varying sizes, we limit the maximum samples per dataset to 20,000 through random sub-sampling without replacement. Target variables are normalized to [0, 10] to standardize metric scales across tasks. We employ a stratified 80/10/10 - train/validation/test split at the dataset level, ensuring representative samples from each dataset in all splits. We use both point prediction metrics (SMAPE, AAD, RMSE, BIAS) and distributional accuracy metrics (CRPS, COVERAGE) as defined in Section~\ref{sec:prelim} and Appendix~\ref{sec:performance_metrics}. CRPS is computed over a set of $Q=200$ quantiles sampled uniformly at random. All metrics are computed on the test split and averaged across datasets.

\noindent\textbf{Baselines:} We compare NIAQUE against traditional tree-based models and deep learning approaches:
a) Tree-based models: XGBoost~\citep{Chen_2016}, LightGBM~\citep{NIPS2017_6449f44a}, and CatBoost~\citep{prokhorenkova2019catboost}
b) Deep learning: 
Transformer encoder with NIAQUE quantile decoder (details in Appendix~\ref{sec:transformer-baseline}). We evaluate three training scenarios: a) Global models (denoted by the -Global suffix): trained jointly on all datasets. b) Domain-specific models (suffix -Domain): trained on datasets from the same domain (e.g., housing, medical). c) Local models (suffix -Local): trained individually per dataset. XGBoost and CatBoost are trained using multi-quantile loss with fixed quantiles, with additional quantiles obtained through linear interpolation. LightGBM, is trained with separate models per quantile. For training global tree-based models, we construct a unified table containing samples from all datasets, filling missing features with NULL values.

\begin{table}[t]
\caption{Performance comparison across all 101 datasets. Lower values are better for all metrics except COVERAGE @ 95 (target: 95). For BIAS, lower absolute values are better.}
\centering
\begin{tabular}{l|cccc|cc}
\toprule
Model           & SMAPE & AAD   & RMSE  & BIAS  & COV@95 & CRPS  \\
\midrule
XGBoost-Global  & 31.4  & 0.574 & 1.056 & -0.15 & 94.6  & 0.636 \\
XGBoost-Local   & 25.6  & 0.433 & 0.883 & -0.03 & 90.8  & 0.334 \\
LightGBM-Global & 27.5  & 0.475 & 0.930 & -0.06 & 94.8  & 0.426 \\
LightGBM-Local  & 25.7  & 0.427 & 0.865 & -0.03 & 91.5  & 0.327 \\
CatBoost-Global & 31.3  & 0.561 & 1.030 & -0.12 & 94.9  & 0.443 \\
CatBoost-Local  & 24.3  & 0.408 & 0.840 & -0.03 & 92.7  & 0.315 \\
\midrule
Transformer-Local  & 26.9  & 0.462 & 0.904 & -0.05 & 93.6  & 0.329 \\
Transformer-Global & 23.1  & 0.383 & 0.806 & -0.01 & 94.6  & 0.272 \\
NIAQUE-Local       & 22.8  & 0.377 & 0.797 & -0.03 & 94.9  & 0.267 \\
NIAQUE-Global      & 22.1  & 0.367 & 0.787 & -0.02 & 94.6  & 0.261 \\
\bottomrule
\end{tabular}%
\label{tab:101-performance}
\end{table}

\noindent\textbf{Implementation Details:} NIAQUE uses encoder and decoder containing 4 residual blocks, 2 layers each with latent dimension $E=1024$ and input embedding size $E_{in}=64$. Training uses Adam optimizer with initial learning rate $10^{-4}$ and batch size 512. The learning rate is reduced by 10× at 500k, 600k, and 700k batches. We apply feature dropout with rate 0.2. Hyperparameters are selected using validation split and metrics are computed on the test split. Training requires approximately 24 hours for NIAQUE and 48 hours for Transformer on 4×V100 GPUs. In comparison, XGBoost training takes about 30 minutes on a one V100 for 3 quantiles, scaling linearly with the number of quantiles. All models are trained on the same train splits, with samples drawn uniformly at random across datasets. During training, quantile values are randomly generated for each instance in a batch.

\subsection{Results} 

\noindent\textbf{Cross-Dataset Learning}. To evaluate NIAQUE's ability to handle large-scale multi-dataset learning, we conduct experiments across all 101 datasets simultaneously. Table~\ref{tab:101-performance} presents results, aggregated across datasets at sample level, comparing global and local training scenarios for various models. NIAQUE-Global achieves the best performance, significantly outperforming both traditional tree-based methods and the Transformer baseline. Notably, while tree-based methods show better performance in local training compared to their global variants (e.g., CatBoost-Local SMAPE: 24.3 vs CatBoost-Global: 31.3), NIAQUE maintains superior performance in both scenarios, with its global model outperforming its local counterpart. Furthermore, NIAQUE maintains reliable uncertainty quantification across all scenarios, with coverage staying close to the target 95\% level and consistently lower CRPS values compared to baselines. These results confirm NIAQUE's capacity to leverage cross-dataset learning effectively, maintaining or even improving performance on individual tasks through robust feature representations that generalize across diverse datasets and domains.

\noindent\textbf{Adaptation to New Tasks on TabRegSet-101}. To evaluate NIAQUE's transfer learning capabilities on unseen tasks, we randomly split our collection of 101 datasets into 80 pretraining datasets and 21 held-out test datasets. We compare two scenarios: training from scratch (NIAQUE-Scratch) and fine-tuning a pretrained model (NIAQUE-Pretrain). The pretrained model is first trained on the 80 datasets and then fine-tuned on each held-out dataset using a 10 times smaller learning rate. To assess the impact of data scarcity, we evaluate both models by varying the fine-tuning data proportion ($p_s$) of the held-out datasets while maintaining constant test sets. Results in Table~\ref{tab:transfer-learning-finetune} demonstrate that:
1) The pretrained model consistently outperforms training from scratch across all metrics.
2) The performance gap widens as training data becomes scarcer (smaller $p_s$).
3) Both models maintain reliable uncertainty estimates, as evidenced by COVERAGE @ 95 values. These results validate that NIAQUE effectively transfers pretrained knowledge to novel regression tasks, with improvements particularly pronounced in low-data scenarios. Note that these results are not directly comparable with those in Table~\ref{tab:101-performance}, as they are based on different dataset splits (21 vs. 101 datasets).

\begin{table}[t]
    \centering
    \caption{Transfer learning results on held-out datasets. $p_s$ represents the proportion of training data used for fine-tuning, ranging from 0.05 (5\%) to 1.0 (100\%). Lower values are better for all metrics except COVERAGE @ 95 (target: 95). For BIAS, lower absolute values are better.
    }
        \begin{tabular}{l|l|ccccc}
        \toprule
         & NIAQUE & $p_s$=0.05 & 0.1 & 0.25 & 0.5 & 1.0 \\
        \midrule
        SMAPE & Scratch & 28.0 & 24.7 & 21.7 & 20.8 & 19.4 \\
              & Pretrain & 23.5 & 21.9 & 20.3 & 18.7 & 17.7 \\
        \midrule
        AAD   & Scratch & 0.71 & 0.60 & 0.56 & 0.54 & 0.49 \\
              & Pretrain & 0.61 & 0.57 & 0.54 & 0.50 & 0.47 \\
        \midrule
        RMSE  & Scratch & 1.23 & 1.10 & 1.06 & 1.04 & 0.96 \\
              & Pretrain & 1.11 & 1.08 & 1.04 & 0.97 & 0.94 \\
        \midrule
        BIAS  & Scratch & -0.06 & -0.04 & -0.04 & 0.02 & -0.04 \\
              & Pretrain & -0.06 & -0.07 & -0.06 & -0.06 & -0.04 \\
        \midrule
        CRPS  & Scratch & 0.488 & 0.423 & 0.392 & 0.383 & 0.351 \\
              & Pretrain & 0.427 & 0.404 & 0.380 & 0.354 & 0.334 \\
        \midrule
        COV@95 & Scratch & 93.3 & 93.0 & 94.4 & 93.1 & 94.4 \\
               & Pretrain & 95.3 & 94.4 & 94.2 & 93.9 & 94.6 \\
        \bottomrule
        \end{tabular}
    \label{tab:transfer-learning-finetune}
\end{table}

\noindent\textbf{Adaptation to New Tasks on Kaggle Competitions}. To validate NIAQUE's practical effectiveness in the wild, we evaluate its performance in recent Kaggle competitions: Regression with an Abalone Dataset~\citep{playground-series-s4e4}, Regression with a Flood Prediction Dataset~\citep{playground-series-s4e5}
Our approach involves two stages: pretraining and fine-tuning. First, we pretrain NIAQUE on TabRegSet-101 using our quantile loss framework. Then, we fine-tune the pretrained model on the competition's training data, optimizing for target metric. To systematically evaluate the impact of proposed pretraining strategy and architecture, we show NIAQUE-Scratch baseline (trained only on competition data, no pretraining), a number of tree-based baselines as well as TabDPT and TabPFN pretrained models. While TabDPT and TabPFN show promising results on small datasets, they face significant scalability challenges—TabPFN is limited to 10,000 samples and TabDPT requires substantial context size reduction for large datasets (details in Appendix~\ref{appendix:tabpfn-tabdpt-baselines}). On the other hand, our approach shows very strong scalability and accuracy results on the competition datasets, outperforming vanilla tree-based models as well as TabDPT and TabPFN baselines. Additional results in Appendices~\ref{sec:kaggle-competitions-abalone},~\ref{sec:kaggle-competitions-flood}
show how our approach, without significant manual interventions, further benefits from advanced automatic feature engineering (OpenFE~\citep{zhang2023openfe}) and ensembling thereby  rivalling results of human competitors. These results are particularly significant given that neural networks were generally considered ineffective for these competitions.

\begin{table}[t]
\centering
\caption{Performance comparison on Kaggle competition datasets: Abalone (RMSLE, lower is better), Flood Prediction (R2, higher is better).
}
\begin{tabular}{lcc}
\toprule
Model & \makecell{Abalone \\ RMLSE} & \makecell{Flood Prediction \\ R2 Score} \\ 
\midrule
XGBoost~\cite{xgboost-tabnet-abalone, xgboost-floodpredict} & 0.15019 & 0.842 \\ 
LightGBM~\cite{lightgbm-abalone, lightgbm-floodpredict} & 0.14914 & 0.766 \\ 
CatBoost~\cite{catboost-abalone, catboost-floodpredict} & 0.14783 & 0.845 \\ 
TabNet~\cite{xgboost-tabnet-abalone} & 0.15481 & 0.842 \\ 
TabDPT~\cite{} & 0.15026 & 0.804 \\ 
TabPFN~\cite{} & 0.15732 & 0.431 \\ 
NIAQUE-Scratch & 0.15047 & 0.865 \\ 
NIAQUE-Pretrain & 0.14808 & 0.867 \\ 
Winner~\citep{first-place-abalone,first-place-flood-prediction} & 0.14374 & 0.869 \\ 
\bottomrule
\end{tabular}
\label{tab:kaggle-competition}
\end{table}

\noindent\textbf{Learned Representations} are studied qualitatively in Fig.~\ref{fig:niaque_interpretability_and_clusters} (left) showing UMAP projections~\citep{mcinnes2018umap} of dataset row embeddings derived from NIAQUE's feature encoder. The encoder maps input features $\vec{x}_i$ to a fixed-dimensional latent space using a prototype-based aggregation mechanism. The resulting UMAP visualization reveals distinct dataset-specific clusters, indicating that NIAQUE learns representations that capture dataset-specific characteristics while maintaining a shared latent space that enables effective transfer learning. 

\noindent\textbf{Feature Importance} is based on the inverse of the average confidence interval derived from feature's marginal distribution, as detailed in Equation~\eqref{eqn:niaque_weights}. Qualitatively,  Fig.~\ref{fig:niaque_interpretability_and_clusters} (right), indicates that features with higher weights (i.e., smaller average confidence intervals) are most critical to prediction accuracy---removing these features significantly increases the AAD metric, whereas eliminating features with lower weights has minimal impact. Quantitatively, two-sided t-test comparing the impact of removing the most vs. least important features (Top-1 vs. Bot-1 AAD values in Fig.~\ref{fig:niaque_interpretability_and_clusters}) ($t = -50.24$, p-value $\approx$ 0) along with the effect size (Cohen’s $d = 0.22$) indicate significant and practically impactful effect across datasets. Additionally, we computed SHAP values with shap.SamplingExplainer across 101 datasets and used them as a reference ranking. For each dataset, we computed the NDCG score between the SHAP ranking and our model-native importance ranking, then averaged the results. The average NDCG was 0.899, while computation time was reduced from ~3.5h × 8 GPUs to ~20s × 1 GPU for all 101 datasets. The NDCG score of 0.9 generally demonstrates very strong ranking alignment with established attribution method at a fraction of the cost.

\begin{figure}[t]
    \centering
    
    \begin{subfigure}[b]{0.58\textwidth}
        \centering
        \includegraphics[width=\textwidth]{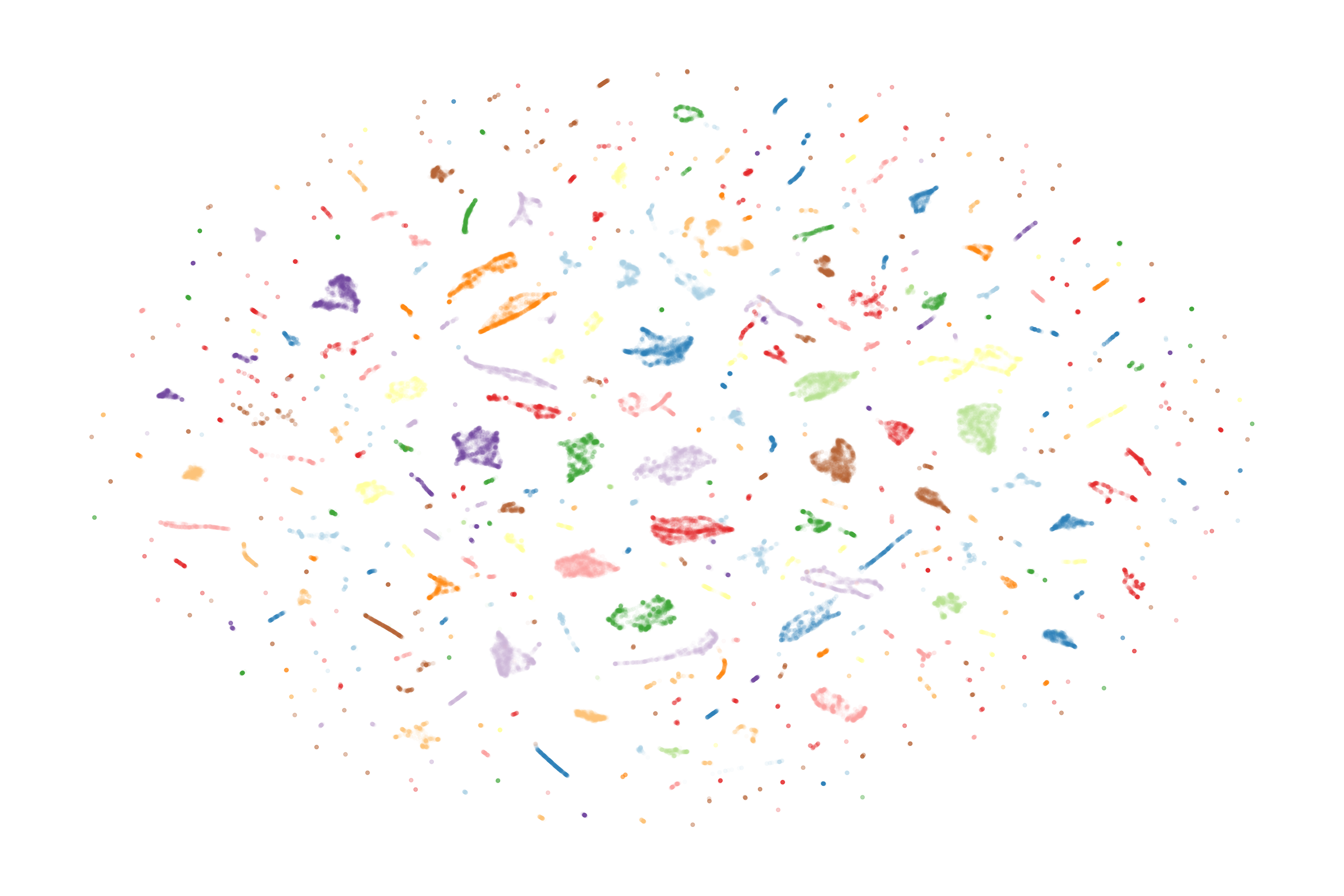}
    \end{subfigure}
    \begin{subfigure}[b]{0.4\textwidth}
        \centering
        \includegraphics[width=\textwidth]{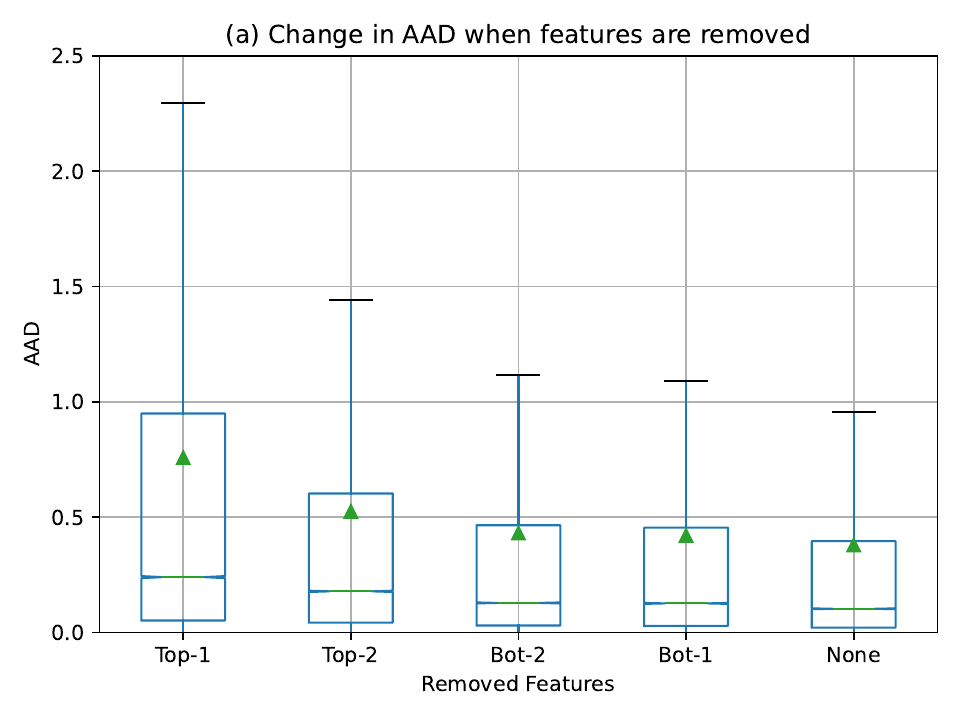}
        \label{fig:interpretability_analysis}
    \end{subfigure}
    \vspace{-0.4cm}
    \caption{UMAP projections of embeddings derived from NIAQUE's feature encoder for each sample, colored by dataset (left). NIAQUE accuracy response to the removal of input features by importance (right). Top-rated features have the greatest impact on AAD degradation when removed.
    }
    \label{fig:niaque_interpretability_and_clusters}
\end{figure}

\noindent\textbf{Ablation Studies} are detailed in Appendices~\ref{sec:transformer-baseline}--~\ref{sec:niaque-ablation-appendix} and validate NIAQUE's architectural choices. Our experiments reveal following key findings. First, the choice of encoder design compared to transformer/attention produces more accurate results. Second, the log-transformation of input values (Equation~\eqref{eqn:log_transformation}) significantly improves both training stability and prediction accuracy. Third, NIAQUE shows robustness to variations in hyperparameters and benefits from increased depth. Finally, including single-feature samples during training enables the interpretability mechanism without compromising the model's prediction accuracy.

\section{Conclusion}

Our study demonstrates that NIAQUE enables effective transfer learning for tabular probabilistic regression, challenging the belief that tabular data resists neural modeling and generalization. Across 101 datasets and a Kaggle case studies, NIAQUE shows strong empirical performance, delivering scalability via a unified architecture, data efficiency through cross-dataset pretraining, and robust generalization across tasks. Its probabilistic formulation further supports uncertainty quantification and feature importance estimation—useful for real-world deployment. While promising, open questions remain regarding the optimal scale of transfer, the design of universally effective feature preprocessing, and the theoretical principles underlying cross-domain generalization. Our findings suggest a path forward for building large-scale, interpretable tabular foundation models.

\clearpage
\bibliographystyle{plainnat}
\bibliography{sample-base}

\appendix

\clearpage
\onecolumn

\section{Performance Metrics} \label{sec:performance_metrics}

Coverage measures empirical calibration of predictive confidence intervals. For confidence level $\alpha \in (0,1)$:
\begin{align}
\coverage(\alpha) = \frac{1}{S} \sum_{i=1}^S \mathds{1}[y_i > \hat y_{i, 0.5-\alpha/2}] \mathds{1}[y_i < \hat y_{i, 0.5+\alpha/2}] \cdot 100\% \nonumber
\end{align}

We employ the following metrics to evaluate the accuracy of point predictions:

1) Symmetric Mean Absolute Percentage Error (sMAPE):
\begin{align}
    \smape = \frac{2}{S} \sum_{i=1}^S \frac{|y_i - \hat y_{i,0.5}|}{|y_i| + |\hat y_{i,0.5}|} \cdot 100\% \,.
\end{align}

2) Average Absolute Deviation (AAD):
\begin{align}
    \aad = \frac{1}{S} \sum_{i=1}^S |y_i - \hat y_{i,0.5}|\,.
\end{align}

3) Bias:
\begin{align}
    \bias = \frac{1}{S} \sum_{i=1}^S  \hat y_{i,0.5} - y_i\,.
\end{align}

4) Root Mean Square Error (RMSE):
\begin{align}
    \rmse = \sqrt{\frac{1}{S} \sum_{i=1}^S (y_i - \hat y_{i,0.5})^2}\,.
\end{align}

5) Root Mean Square Logarithmic Error (RMSLE):
\begin{align}
    \rmsle = \sqrt{\frac{1}{S} \sum_{i=1}^S (\log(y_i + 1) - \log(\hat y_{i,0.5} + 1))^2}\,.
\end{align}

\section{Proof of Theorem~\ref{thm:crps_minimizer}} \label{sec:proof_of_theorem1}

\begin{thmn} 
Let $F$ be a probability measure over variable $y$ such that inverse $F^{-1}$ exists and let $P_{y,\vec{x}}$ be the joint probability measure of variables $\vec{x}, y$. Then the expected loss, $\mathbb{E}\, \rho(y, F^{-1}(q))$, is minimized if and only if:
\begin{equation}
F = P_{y|\vec{x}}\,.
\end{equation}
Additionally:
\begin{equation}
\min_{F} \mathbb{E}\, \rho(y, F^{-1}(q)) = \mathbb{E}_{\nbeatsinput} \frac{1}{2} \int_\Re P_{y|\nbeatsinput}(z)(1 - P_{y|\nbeatsinput}(z)) \textrm{d}z\,.
\end{equation}
\end{thmn}
\begin{proof}
First, combining~(\ref{eqn:crps_expected}) with the L2 representation of CRPS~\eqref{eqn:crps_theoretical} we can write:
\begin{align}
\mathbb{E}\, \rho(y, F^{-1}(q)) 
&= \mathbb{E}_{\nbeatsinput, y} \frac{1}{2}\int_\Re \left(F(z)- \mathds{1}_{\{z \geq y\}} \right)^{2}\textrm{d}z \label{eqn:crps_l2_expectation} \\
&= \mathbb{E}_{\nbeatsinput} \mathbb{E}_{y|\nbeatsinput} \frac{1}{2}\int_\Re F^2(z) - 2F(z)\mathds{1}_{\{z \geq y\}} + \mathds{1}_{\{z \geq y\}} \textrm{d}z \\
&= \mathbb{E}_{\nbeatsinput} \frac{1}{2}\int_\Re F^2(z) - 2F(z)\mathbb{E}_{y|\nbeatsinput}\mathds{1}_{\{z \geq y\}} + \mathbb{E}_{y|\nbeatsinput}\mathds{1}_{\{z \geq y\}} \textrm{d}z \\
&= \mathbb{E}_{\nbeatsinput} \frac{1}{2}\int_\Re F^2(z) - 2F(z)P_{y|\nbeatsinput}(z) + P_{y|\nbeatsinput}(z) \textrm{d}z.
\end{align}
Here we used the law of total expectation and Fubini theorem to exchange the order of integration and then used the fact that $\mathbb{E}_{y|\nbeatsinput}\mathds{1}_{\{z \geq y\}} = P_{y|\nbeatsinput}(z)$. Completing the square we further get:
\begin{align}
\mathbb{E}\, \rho(y, F^{-1}(q)) 
&= \mathbb{E}_{\nbeatsinput} \frac{1}{2}\int_\Re F^2(z) - 2F(z)P_{y|\nbeatsinput}(z) + P_{y|\nbeatsinput}(z) + P^2_{y|\nbeatsinput}(z) - P^2_{y|\nbeatsinput}(z) \textrm{d}z \\
&= \mathbb{E}_{\nbeatsinput} \frac{1}{2}\int_\Re (F(z) - P_{y|\nbeatsinput}(z))^2 + P_{y|\nbeatsinput}(z) - P^2_{y|\nbeatsinput}(z) \textrm{d}z
\end{align}
$F = P_{y|\nbeatsinput}$ is clearly the unique minimizer of the last expression since $\int_\Re (F(z) - P_{y|\nbeatsinput}(z))^2\textrm{d}z > 0, \forall F \neq P_{y|\nbeatsinput}$.
\end{proof}

\clearpage
\section{TabRegSet-101 Details} \label{sec:lprm-101}

We introduce TabRegSet-101 (Tabular Regression Set 101), a curated collection of 101 publicly available regression datasets gathered from UCI~\cite{kelly2017uci}, Kaggle~\cite{kaggle2024kaggle}, PMLB~\cite{romano2021pmlb,olson2017pmlb}, OpenML~\cite{vanschoren2013openml}, and KEEL~\cite{fdez2011keel}. These datasets span diverse domains, including \textit{Housing and Real Estate}, \textit{Energy and Efficiency}, \textit{Retail and Sales}, \textit{Computer Systems}, \textit{Physics Models} and \textit{Medicine} and exhibit different characteristics in terms of sample size and feature dimensionality (Fig.~\ref{fig:dataset_stats_appendix}). The datasets, along with their sample count, number of variables and source information are listed in Table~\ref{table:lprm101}.

\begin{figure}[t]
    \centering
    \hspace{-0.91cm}
    \includegraphics[width=1.0\linewidth]{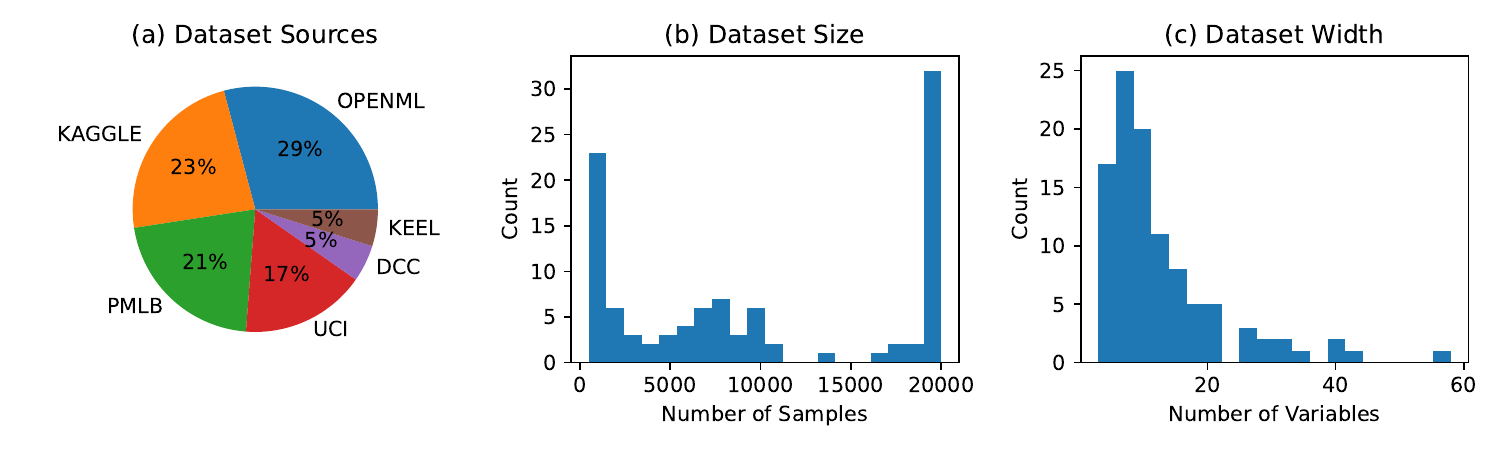}
    \caption{Statistics of the evaluation dataset: (a) distribution by source, (b) dataset sizes, and (c) feature counts.}
    \label{fig:dataset_stats_appendix}
\end{figure}

We focus specifically on the regression task in which the target variable is continuous or, if it has limited number of levels, these are ordered such as student exam scores or wine quality. The target variable in each dataset is normalized to the [0, 10] range and the independent variables are used as is, raw. The target variable scaling is applied to equalize the contributions of the evaluation metrics from each dataset. Datasets have variable number of samples, the lowest being just below 1000. For very large datasets we limit the number of samples used in our benchmark to be 20,000 by subsampling uniformly at random. This allows us (i) to model imbalance, and at the same time (ii) avoid the situation in which a few large datasets could completely dominate the training and evaluation of the model. The distribution of datasets by source, number of samples and number of variables is shown in Figure~\ref{fig:dataset_stats_appendix}.

\clearpage
\onecolumn

\newcolumntype{S}{>{\footnotesize}l}
\newcolumntype{F}[1]{>{\fontsize{8}{9}\selectfont}p{#1}}

\begin{longtable}{lSrrSF{7.6cm}}
\caption{The list of datasets comprising TabRegSet-101 benchmark} \label{table:lprm101} \\
\toprule
id & Name & n\_samples & n\_vars & source & URL \\
\midrule
\endfirsthead
\caption[]{The list of datasets comprising TabRegSet-101 benchmark} \\
\toprule
 & name & n\_samples & n\_vars & source & url \\
\midrule
\endhead
\midrule
\multicolumn{6}{r}{Continued on next page} \\
\midrule
\endfoot
\bottomrule
\endlastfoot
0 & Abalone & 4177 & 7 & uci & \url{https://archive.ics.uci.edu/static/public/1/data.csv} \\
1 & Student\_Performance & 649 & 29 & uci & \url{https://archive.ics.uci.edu/static/public/320/data.csv} \\
2 & Infrared\_Thermography\_Temperature & 1020 & 32 & uci & \url{https://archive.ics.uci.edu/static/public/925/data.csv} \\
3 & Parkinsons\_Telemonitoring & 5875 & 18 & uci & \url{https://archive.ics.uci.edu/static/public/189/data.csv} \\
4 & Energy\_Efficiency & 768 & 7 & uci & \url{https://archive.ics.uci.edu/static/public/242/data.csv} \\
5 & 1027\_ESL & 488 & 3 & pmlb & \url{https://github.com/EpistasisLab/penn-ml-benchmarks/raw/master/datasets/1027_ESL/1027_ESL.tsv.gz} \\
6 & 1028\_SWD & 1000 & 9 & pmlb & \url{https://github.com/EpistasisLab/penn-ml-benchmarks/raw/master/datasets/1028_SWD/1028_SWD.tsv.gz} \\
7 & 1029\_LEV & 1000 & 3 & pmlb & \url{https://github.com/EpistasisLab/penn-ml-benchmarks/raw/master/datasets/1029_LEV/1029_LEV.tsv.gz} \\
8 & 1030\_ERA & 1000 & 3 & pmlb & \url{https://github.com/EpistasisLab/penn-ml-benchmarks/raw/master/datasets/1030_ERA/1030_ERA.tsv.gz} \\
9 & 1199\_BNG\_echoMonths & 17496 & 8 & pmlb & \url{https://github.com/EpistasisLab/penn-ml-benchmarks/raw/master/datasets/1199_BNG_echoMonths/1199_BNG_echoMonths.tsv.gz} \\
10 & 197\_cpu\_act & 8192 & 20 & pmlb & \url{https://github.com/EpistasisLab/penn-ml-benchmarks/raw/master/datasets/197_cpu_act/197_cpu_act.tsv.gz} \\
11 & 225\_puma8NH & 8192 & 7 & pmlb & \url{https://github.com/EpistasisLab/penn-ml-benchmarks/raw/master/datasets/225_puma8NH/225_puma8NH.tsv.gz} \\
12 & 227\_cpu\_small & 8192 & 11 & pmlb & \url{https://github.com/EpistasisLab/penn-ml-benchmarks/raw/master/datasets/227_cpu_small/227_cpu_small.tsv.gz} \\
13 & 294\_satellite\_image & 6435 & 35 & pmlb & \url{https://github.com/EpistasisLab/penn-ml-benchmarks/raw/master/datasets/294_satellite_image/294_satellite_image.tsv.gz} \\
14 & 344\_mv & 20000 & 9 & pmlb & \url{https://github.com/EpistasisLab/penn-ml-benchmarks/raw/master/datasets/344_mv/344_mv.tsv.gz} \\
15 & 503\_wind & 6574 & 13 & pmlb & \url{https://github.com/EpistasisLab/penn-ml-benchmarks/raw/master/datasets/503_wind/503_wind.tsv.gz} \\
16 & 529\_pollen & 3848 & 3 & pmlb & \url{https://github.com/EpistasisLab/penn-ml-benchmarks/raw/master/datasets/529_pollen/529_pollen.tsv.gz} \\
17 & 537\_houses & 20000 & 7 & pmlb & \url{https://github.com/EpistasisLab/penn-ml-benchmarks/raw/master/datasets/537_houses/537_houses.tsv.gz} \\
18 & 547\_no2 & 500 & 6 & pmlb & \url{https://github.com/EpistasisLab/penn-ml-benchmarks/raw/master/datasets/547_no2/547_no2.tsv.gz} \\
19 & 564\_fried & 20000 & 9 & pmlb & \url{https://github.com/EpistasisLab/penn-ml-benchmarks/raw/master/datasets/564_fried/564_fried.tsv.gz} \\
20 & 595\_fri\_c0\_1000\_10 & 1000 & 9 & pmlb & \url{https://github.com/EpistasisLab/penn-ml-benchmarks/raw/master/datasets/595_fri_c0_1000_10/595_fri_c0_1000_10.tsv.gz} \\
21 & 593\_fri\_c1\_1000\_10 & 1000 & 9 & pmlb & \url{https://github.com/EpistasisLab/penn-ml-benchmarks/raw/master/datasets/593_fri_c1_1000_10/593_fri_c1_1000_10.tsv.gz} \\
22 & 1193\_BNG\_lowbwt & 20000 & 8 & pmlb & \url{https://github.com/EpistasisLab/penn-ml-benchmarks/raw/master/datasets/1193_BNG_lowbwt/1193_BNG_lowbwt.tsv.gz} \\
23 & 1201\_BNG\_breastTumor & 20000 & 8 & pmlb & \url{https://github.com/EpistasisLab/penn-ml-benchmarks/raw/master/datasets/1201_BNG_breastTumor/1201_BNG_breastTumor.tsv.gz} \\
24 & 1203\_BNG\_pwLinear & 20000 & 9 & pmlb & \url{https://github.com/EpistasisLab/penn-ml-benchmarks/raw/master/datasets/1203_BNG_pwLinear/1203_BNG_pwLinear.tsv.gz} \\
25 & 215\_2dplanes & 20000 & 9 & pmlb & \url{https://github.com/EpistasisLab/penn-ml-benchmarks/raw/master/datasets/215_2dplanes/215_2dplanes.tsv.gz} \\
26 & 218\_house\_8L & 20000 & 7 & pmlb & \url{https://github.com/EpistasisLab/penn-ml-benchmarks/raw/master/datasets/218_house_8L/218_house_8L.tsv.gz} \\
27 & QsarFishToxicity & 908 & 5 & uci & \url{https://archive.ics.uci.edu/static/public/504/qsar+fish+toxicity.zip} \\
28 & concrete\_compressive\_strength & 1030 & 7 & uci & \url{https://archive.ics.uci.edu/static/public/165/concrete+compressive+strength.zip} \\
29 & PRODUCTIVITY & 1197 & 12 & uci & \url{https://archive.ics.uci.edu/static/public/597/productivity+prediction+of+garment+employees.zip} \\
30 & CCPP & 9568 & 3 & uci & \url{https://archive.ics.uci.edu/static/public/294/combined+cycle+power+plant.zip} \\
31 & AIRFOIL & 1503 & 4 & uci & \url{https://archive.ics.uci.edu/static/public/291/airfoil+self+noise.zip} \\
32 & TETOUAN & 20000 & 6 & uci & \url{https://archive.ics.uci.edu/static/public/849/power+consumption+of+tetouan+city.zip} \\
33 & BIAS\_CORRECTION & 7725 & 22 & uci & \url{https://archive.ics.uci.edu/static/public/514/bias+correction+of+numerical+prediction+model+temperature+forecast.zip} \\
34 & APARTMENTS & 10000 & 10 & uci & \url{https://archive.ics.uci.edu/static/public/555/apartment+for+rent+classified.zip} \\
35 & MedicalCost & 1338 & 5 & kaggle & \url{kaggle datasets download -d mirichoi0218/insurance} \\
36 & Vehicle & 2059 & 18 & kaggle & \url{kaggle datasets download -d nehalbirla/vehicle-dataset-from-cardekho} \\
37 & LifeExpectancy & 2928 & 18 & kaggle & \url{kaggle datasets download -d kumarajarshi/life-expectancy-who} \\
38 & CalHousing & 20000 & 7 & dcc & \url{https://www.dcc.fc.up.pt/~ltorgo/Regression/cal_housing.tgz} \\
39 & Ailerons & 7154 & 39 & dcc & \url{https://www.dcc.fc.up.pt/~ltorgo/Regression/ailerons.tgz} \\
40 & DeltaElevators & 9517 & 5 & dcc & \url{https://www.dcc.fc.up.pt/~ltorgo/Regression/delta_elevators.tgz} \\
41 & Pole & 10000 & 25 & dcc & \url{https://www.dcc.fc.up.pt/~ltorgo/Regression/pol.tgz} \\
42 & Kinematics & 8192 & 7 & dcc & \url{https://www.dcc.fc.up.pt/~ltorgo/Regression/kinematics.tar.gz} \\
43 & BigMartSales & 8523 & 10 & kaggle & \url{kaggle datasets download -d brijbhushannanda1979/bigmart-sales-data} \\
44 & VideoGameSales & 16598 & 3 & kaggle & \url{kaggle datasets download -d gregorut/videogamesales} \\
45 & NewsPopularity & 20000 & 58 & uci & \url{https://archive.ics.uci.edu/static/public/332/online+news+popularity.zip} \\
46 & Wizmir & 1461 & 8 & keel & \url{https://sci2s.ugr.es/keel/dataset/data/regression/wizmir.zip} \\
47 & Ele2 & 1056 & 3 & keel & \url{https://sci2s.ugr.es/keel/dataset/data/regression/ele-2.zip} \\
48 & Treasury & 1049 & 14 & keel & \url{https://sci2s.ugr.es/keel/dataset/data/regression/treasury.zip} \\
49 & Mortgage & 1049 & 14 & keel & \url{https://sci2s.ugr.es/keel/dataset/data/regression/mortgage.zip} \\
50 & Laser & 993 & 3 & keel & \url{https://sci2s.ugr.es/keel/dataset/data/regression/laser.zip} \\
51 & SpaceGa & 3107 & 5 & openml & \url{https://www.openml.org/data/download/52619/space_ga.arff} \\
52 & VisualizingSoil & 8641 & 3 & openml & \url{https://www.openml.org/data/download/52988/visualizing_soil.arff} \\
53 & Diamonds & 20000 & 8 & openml & \url{https://www.openml.org/data/download/21792853/dataset.arff} \\
54 & TitanicFare & 1307 & 6 & openml & \url{https://www.openml.org/data/download/20649205/file277c5e2b70e8.arff} \\
55 & Sulfur & 10081 & 5 & openml & \url{https://www.openml.org/data/download/2095629/phpBXEqg1.arff} \\
56 & Debutanizer & 2394 & 6 & openml & \url{https://www.openml.org/data/download/2096280/phpWT77lf.arff} \\
57 & Fardamento & 6277 & 5 & openml & \url{https://www.openml.org/data/download/21854531/fardamento_saidas_19_20a20maio.arff} \\
58 & ProteinTertiary & 20000 & 8 & openml & \url{https://api.openml.org/data/download/22111827/file22f167620a212.arff} \\
59 & BrazilianHouses & 10692 & 7 & openml & \url{https://api.openml.org/data/download/22111854/file22f1627e4a960.arff} \\
60 & Cps88Wages & 20000 & 5 & openml & \url{https://api.openml.org/data/download/22111848/file22f161d4b5556.arff} \\
61 & CPMP-2015 & 2108 & 25 & openml & \url{https://www.openml.org/data/download/21377442/file16a868cf35f5.arff} \\
62 & NASA-PHM2008 & 20000 & 16 & openml & \url{https://www.openml.org/data/download/22045221/dataset.arff} \\
63 & Wind & 6574 & 12 & openml & \url{https://www.openml.org/data/download/52615/wind.arff} \\
64 & NewFuelCar & 20000 & 17 & openml & \url{https://www.openml.org/data/download/21230500/pruebaconvonline.csv.arff} \\
65 & MiamiHousing & 13932 & 14 & openml & \url{https://www.openml.org/data/download/22047757/miami2016.arff} \\
66 & BlackFriday & 20000 & 8 & openml & \url{https://www.openml.org/data/download/21230845/file639340bd9ca9.arff} \\
67 & IEEE80211aaGATS & 5296 & 28 & openml & \url{https://www.openml.org/data/download/22101884/dataset.arff} \\
68 & Yprop41 & 8885 & 41 & openml & \url{https://api.openml.org/data/download/22111920/dataset.arff} \\
69 & Sarcos & 20000 & 20 & openml & \url{https://api.openml.org/data/download/22111840/file22f166a1669bb.arff} \\
70 & ZurichDelays & 20000 & 16 & openml & \url{https://www.openml.org/data/download/21854423/file86eb92864fd.arff} \\
71 & 1000-Cameras & 1015 & 13 & openml & \url{https://www.openml.org/data/download/22102539/dataset.arff} \\
72 & GridStability & 10000 & 11 & openml & \url{https://api.openml.org/data/download/22111837/file22f1652de1c8a.arff} \\
73 & PumaDyn32nh & 8192 & 31 & openml & \url{https://api.openml.org/data/download/22111845/file22f161b261f3b.arff} \\
74 & Fifa & 19178 & 27 & openml & \url{https://api.openml.org/data/download/22111894/file10aca711933d5.arff} \\
75 & WhiteWine & 4898 & 10 & openml & \url{https://api.openml.org/data/download/22111835/file22f16150a82cd.arff} \\
76 & RedWine & 1599 & 10 & openml & \url{https://api.openml.org/data/download/22111836/file22f162b311c38.arff} \\
77 & FpsBenchmark & 20000 & 42 & openml & \url{https://api.openml.org/data/download/22111856/file22f1639d20997.arff} \\
78 & KingCountyHousing & 20000 & 20 & openml & \url{https://api.openml.org/data/download/22111853/file22f167bd414f1.arff} \\
79 & AvocadoPrices & 18249 & 12 & kaggle & \url{kaggle datasets download -d neuromusic/avocado-prices} \\
80 & Transcoding & 20000 & 18 & uci & \url{https://archive.ics.uci.edu/static/public/335/online+video+characteristics+and+transcoding+time+dataset.zip} \\
81 & house\_16H & 20000 & 15 & openml & \url{https://www.openml.org/data/download/52752/house_16H.arff} \\
82 & Sales & 10738 & 13 & openml & \url{https://www.openml.org/data/download/21756753/dataset.arff} \\
83 & WalmartSales & 6435 & 8 & kaggle & \url{kaggle datasets download -d mikhail1681/walmart-sales} \\
84 & UsedCar & 6019 & 11 & kaggle & \url{kaggle datasets download -d nitishjolly/used-car-price-prediction} \\
85 & HouseRent & 4746 & 11 & kaggle & \url{kaggle datasets download -d iamsouravbanerjee/house-rent-prediction-dataset} \\
86 & LaptopPrice & 1273 & 15 & kaggle & \url{kaggle datasets download -d ehtishamsadiq/uncleaned-laptop-price-dataset} \\
87 & UberFare & 20000 & 8 & kaggle & \url{kaggle datasets download -d yasserh/uber-fares-dataset} \\
88 & Co2Emission & 7385 & 10 & kaggle & \url{kaggle datasets download -d debajyotipodder/co2-emission-by-vehicles} \\
89 & SongPopularity & 18835 & 12 & kaggle & \url{kaggle datasets download -d yasserh/song-popularity-dataset} \\
90 & Cars & 20000 & 8 & kaggle & \url{kaggle datasets download -d aishwaryamuthukumar/cars-dataset-audi-bmw-ford-hyundai-skoda-vw} \\
91 & GemstonePrice & 20000 & 8 & kaggle & \url{kaggle datasets download -d colearninglounge/gemstone-price-prediction} \\
92 & LoanAmount & 20000 & 20 & kaggle & \url{kaggle datasets download -d phileinsophos/predict-loan-amount-data} \\
93 & SaudiArabiaCars & 5507 & 10 & kaggle & \url{kaggle datasets download -d turkibintalib/saudi-arabia-used-cars-dataset} \\
94 & GpuKernelPerformance & 20000 & 13 & kaggle & \url{kaggle datasets download -d rupals/gpu-runtime} \\
95 & AmericanHousePrices & 20000 & 10 & kaggle & \url{kaggle datasets download -d jeremylarcher/american-house-prices-and-demographics-of-top-cities} \\
96 & KindleBooks & 20000 & 12 & kaggle & \url{kaggle datasets download -d asaniczka/amazon-kindle-books-dataset-2023-130k-books} \\
97 & BookSales & 1070 & 8 & kaggle & \url{kaggle datasets download -d thedevastator/books-sales-and-ratings} \\
98 & CapitalGain & 20000 & 12 & kaggle & \url{kaggle datasets download -d minnieliang/adult-data} \\
99 & MarketingCampaign & 2976 & 14 & kaggle & \url{kaggle datasets download -d ahmadazari/marketing-campaign-data} \\
100 & CampaignUplift & 2000 & 9 & kaggle & \url{kaggle datasets download -d hwwang98/software-usage-promotion-campaign-uplift-model} \\
\end{longtable}
\clearpage
\onecolumn

\section{Supplementary Results}\label{sec:supplimentary-results}

\subsection{Domain-Specific Results}
While the main paper focuses on the House Price Prediction domain with key comparisons, here we present comprehensive results for both House Price Prediction and Energy and Efficiency domains. 

\noindent\textbf{House Price Prediction Domain:} Table~\ref{tab:house-performance-full} presents the complete performance comparison for HouseRent dataset~\cite{Banerjee_House_Rent_Prediction} (4.7K samples, 11 features).

\begin{table}[t]
\caption{Performance comparison on HouseRent dataset across point prediction accuracy (SMAPE, AAD, RMSE, BIAS) and distributional accuracy (COV@95, CRPS) metrics. Lower values are better for all metrics except COV@95, where values closer to 95 are optimal. The results with 95\% confidence intervals derived from multiple random seed runs.}
\begin{tabular}{l|cccccc}
\toprule
Model & SMAPE & AAD & RMSE & BIAS & COV@95 & CRPS \\
\midrule
XGBoost-Local & $34.1\pm 0.1$ & $0.028\pm 0.001$ & $0.066\pm 0.004$ & $-0.002\pm 0.01$ & $86.5\pm 0.2$ & $0.022\pm 0.001$ \\
XGBoost-Domain & $32.9\pm 0.2$ & $0.031\pm 0.002$ & $0.073\pm 0.005$ & $0.004\pm 0.01$ & $88.8\pm 0.2$ & $0.026\pm 0.002$ \\
XGBoost-Global & $35.4\pm 4.4$ & $0.033\pm 0.100$ & $0.075\pm 0.143$ & $0.011\pm 0.05$ & $62.5\pm 0.3$ & $0.033\pm 0.165$ \\
\midrule
LightGBM-Local & $31.4\pm 0.2$ & $0.028\pm 0.001$ & $0.068\pm 0.001$ & $-0.01\pm 0.01$ & $92.5\pm 0.7$ & $0.021\pm 0.001$ \\
LightGBM-Domain & $33.6\pm 1.3$ & $0.029\pm 0.001$ & $0.07\pm 0.001$ & $-0.01\pm 0.01$ & $91.8\pm 1.3$ & $0.024\pm 0.001$ \\
LightGBM-Global & $37.1\pm 0.2$ & $0.036\pm 0.001$ & $0.085\pm 0.003$ & $-0.01\pm 0.01$ & $94.4\pm 0.7$ & $0.028\pm 0.001$ \\
\midrule
CatBoost-Local & $31.5\pm 0.1$ & $0.028\pm 0.001$ & $0.068\pm 0.003$ & $-0.007\pm 0.01$ & $92.8\pm 0.1$ & $0.022\pm 0.001$ \\
CatBoost-Domain & $32.4\pm 0.2$ & $0.028\pm 0.002$ & $0.072\pm 0.004$ & $-0.011\pm 0.01$ & $94.1\pm 0.2$ & $0.022\pm 0.002$ \\
CatBoost-Global & $50.3\pm 0.2$ & $0.044\pm 0.006$ & $0.104\pm 0.009$ & $-0.015\pm 0.02$ & $88.4\pm 0.2$ & $0.032\pm 0.004$ \\
\midrule
Transformer-Local & $33.5\pm 0.3$ & $0.030\pm 0.008$ & $0.070\pm 0.015$ & $-0.008\pm 0.01$ & $93.6\pm 0.3$ & $0.025\pm 0.005$ \\
Transformer-Domain & $32.8\pm 0.3$ & $0.029\pm 0.008$ & $0.069\pm 0.015$ & $-0.007\pm 0.01$ & $94.0\pm 0.3$ & $0.023\pm 0.005$ \\
Transformer-Global & $32.3\pm 0.3$ & $0.028\pm 0.008$ & $0.068\pm 0.015$ & $-0.007\pm 0.01$ & $93.5\pm 0.3$ & $0.020\pm 0.005$ \\
\midrule
NIAQUE-Local & $32.0\pm 0.4$ & $0.028\pm 0.012$ & $0.067\pm 0.019$ & $-0.007\pm 0.01$ & $96.0\pm 0.4$ & $0.019\pm 0.011$ \\
NIAQUE-Domain & $30.4\pm 0.2$ & $0.026\pm 0.006$ & $0.063\pm 0.010$ & $-0.002\pm 0.01$ & $93.5\pm 0.3$ & $0.018\pm 0.006$ \\
NIAQUE-Global & $30.3\pm 0.1$ & $0.026\pm 0.002$ & $0.067\pm 0.005$ & $-0.005\pm 0.01$ & $93.3\pm 0.2$ & $0.019\pm 0.002$ \\
\bottomrule
\end{tabular}
\label{tab:house-performance-full}
\end{table}

\noindent\textbf{Energy and Efficiency Domain:} We analyze the Wind dataset~\cite{OpenML_Wind_Dataset} (6.5K samples, 12 features) as our representative case study. Similar to HouseRent, this dataset was chosen for its moderate sample size and limited feature dimensionality, providing a balanced evaluation ground. Table~\ref{tab:wind-performance} shows the performance comparison between NIAQUE and baselines across different training scenarios. The results reinforce our main findings: NIAQUE maintains consistent performance across local, domain, and global training settings (RMSE: 0.747-0.760), while traditional models like CatBoost show significant degradation in global settings (RMSE increases from 0.754 to 1.068).

\begin{table}[t]
\caption{Performance comparison on Wind dataset across point prediction accuracy (SMAPE, AAD, RMSE, BIAS) and distributional accuracy (COV@95, CRPS) metrics. Lower values are better for all metrics except COV@95, where values closer to 95 are optimal. The results with 95\% confidence intervals derived from multiple random seed runs.}
\begin{tabular}{l|cccccc}
\toprule
Model & SMAPE & AAD & RMSE & BIAS & COV@95 & CRPS \\
\midrule
XGBoost-Local & $19.4\pm 0.1$ & $0.607\pm 0.001$ & $0.779\pm 0.004$ & $0.044\pm 0.01$ & $64.9\pm 0.2$ & $0.473\pm 0.001$ \\
XGBoost-Domain & $19.3\pm 0.2$ & $0.603\pm 0.002$ & $0.779\pm 0.005$ & $0.016\pm 0.01$ & $96.0\pm 0.2$ & $0.462\pm 0.002$ \\
XGBoost-Global & $20.0\pm 4.4$ & $0.627\pm 0.100$ & $0.811\pm 0.143$ & $0.010\pm 0.05$ & $97.0\pm 0.3$ & $0.527\pm 0.165$ \\
\midrule
LightGBM-Local & $18.8\pm 0.2$ & $0.589\pm 0.01$ & $0.753\pm 0.007$ & $0.04\pm 0.01$ & $91.5\pm 1.0$ & $0.429\pm 0.003$ \\
LightGBM-Domain & $18.9\pm 0.1$ & $0.592\pm 0.004$ & $0.766\pm 0.006$ & $0.05\pm 0.01$ & $93.8\pm 0.6$ & $0.446\pm 0.002$ \\
LightGBM-Global & $20.8\pm 0.1$ & $0.657\pm 0.002$ & $0.854\pm 0.001$ & $0.02\pm 0.01$ & $97.8\pm 0.3$ & $0.599\pm 0.032$ \\
\midrule
CatBoost-Local & $18.9\pm 0.1$ & $0.590\pm 0.001$ & $0.754\pm 0.003$ & $0.046\pm 0.01$ & $93.6\pm 0.1$ & $0.424\pm 0.001$ \\
CatBoost-Domain & $19.8\pm 0.2$ & $0.617\pm 0.002$ & $0.796\pm 0.004$ & $0.038\pm 0.01$ & $94.8\pm 0.2$ & $0.454\pm 0.002$ \\
CatBoost-Global & $26.3\pm 0.2$ & $0.832\pm 0.006$ & $1.068\pm 0.009$ & $0.027\pm 0.02$ & $98.3\pm 0.2$ & $0.722\pm 0.004$ \\
\midrule
Transformer-Local & $18.5\pm 0.3$ & $0.575\pm 0.008$ & $0.740\pm 0.015$ & $0.025\pm 0.01$ & $94.2\pm 0.3$ & $0.410\pm 0.005$ \\
Transformer-Domain & $18.5\pm 0.3$ & $0.573\pm 0.008$ & $0.736\pm 0.015$ & $0.022\pm 0.01$ & $94.3\pm 0.3$ & $0.405\pm 0.005$ \\
Transformer-Global & $18.4\pm 0.3$ & $0.572\pm 0.008$ & $0.733\pm 0.015$ & $0.020\pm 0.01$ & $94.5\pm 0.3$ & $0.401\pm 0.005$ \\
\midrule
NIAQUE-Local & $18.8\pm 0.4$ & $0.582\pm 0.012$ & $0.747\pm 0.019$ & $0.045\pm 0.01$ & $95.7\pm 0.4$ & $0.407\pm 0.011$ \\
NIAQUE-Domain & $18.7\pm 0.2$ & $0.585\pm 0.006$ & $0.760\pm 0.010$ & $0.027\pm 0.01$ & $95.4\pm 0.3$ & $0.415\pm 0.006$ \\
NIAQUE-Global & $19.0\pm 0.1$ & $0.587\pm 0.002$ & $0.755\pm 0.005$ & $0.031\pm 0.01$ & $93.0\pm 0.2$ & $0.412\pm 0.002$ \\
\bottomrule
\end{tabular}
\label{tab:wind-performance}
\end{table}

Notably, NIAQUE's performance in the Energy and Efficiency domain exhibits similar patterns to those observed in the House Price Prediction domain. The model maintains reliable uncertainty quantification (coverage near 95\%) and demonstrates effective knowledge transfer in domain-specific training, achieving comparable or better performance than local training. These results further support our conclusion about NIAQUE's ability to effectively leverage cross-dataset learning while maintaining robust performance across different domains.

\subsection{Adaptation to New Tasks: Kaggle Competitions} \label{sec:kaggle-competitions}

\begin{table}[t]
\centering
\caption{Performance comparison on Kaggle competition dataset. Lower values are better. 
Baseline results are adopted from publicly shared notebooks and discussion forums in the competition~\cite{playground-series-s4e4}. Rank represents the position of the various methods on private leaderboard.}
\begin{tabular}{lccc}
\toprule
Model & Feature Engineering & RMLSE & Rank \\
\midrule
XGBoost~\cite{xgboost-tabnet-abalone} & None & 0.15019 & 1615 \\
LightGBM~\cite{lightgbm-abalone} & None & 0.14914 & 1464 \\
CatBoost~\cite{catboost-abalone} & None & 0.14783 & 1064 \\
TabNet~\cite{xgboost-tabnet-abalone} & None & 0.15481 & 2047 \\
TabDPT & None & 0.15026 & 1623 \\
TabDPT & OpenFE & 0.14751 & 924 \\
TabPFN & None & 0.15732 & 2132 \\
TabPFN & OpenFE & 0.14922 & 1478 \\
NIAQUE-Scratch & None & 0.15047 & 1646\\
NIAQUE-Pretrain-100 & None & 0.14823 & 1232\\
NIAQUE-Pretrain-full & None & 0.14808 & 1178 \\
NIAQUE-Pretrain-full & OpenFE & 0.14556 & 304 \\
NIAQUE-Ensemble & OpenFE & 0.14423 & 8 \\
Winner~\cite{first-place-abalone} & Manual & 0.14374 & 1 \\
\bottomrule
\end{tabular}
\label{tab:kaggle-competition}
\end{table}

\subsubsection{Abalone Competition} \label{sec:kaggle-competitions-abalone}

To validate NIAQUE's practical effectiveness, we evaluate its performance in recent Kaggle competitions. Abalone~\cite{playground-series-s4e4}, focuses on the Abalone age prediction task. This competition, with 2,700 participants and over 20,000 submissions, provides an excellent real-world benchmark, particularly as neural networks are widely reported to underperform in it, compared to traditional boosted tree methods.

\noindent\textbf{Methodology:} Our approach involves two stages: pretraining and fine-tuning. First, we pretrain NIAQUE on TabRegSet-101 (which includes the original UCI Abalone dataset~\cite{Dua2019Abalone}) using our quantile loss framework. Then, we fine-tune the pretrained model on the competition's training data, optimizing for RMSLE metric as defined in Section~\ref{sec:prelim}. The dataset size is $90,614$ training samples and $60,410$ test samples~\cite{playground-series-s4e4}. To systematically evaluate the impact of transfer learning, we implement three variants: NIAQUE-Scratch (trained only on competition data, no pretraining), NIAQUE-Pretrain-100 (pretrained on TabRegSet-101, excluding Abalone dataset), and NIAQUE-Pretrain-full (pretrained on full TabRegSet-101). Note that the competition dataset is synthetic, generated using a deep learning model trained on the original Abalone dataset, and the competition explicitly encourages the use of the original dataset. Finally, we measure the effect of automated feature engineering on our model using OpenFE~\cite{zhang2023openfe}, an automated feature engineering framework, and explore ensemble strategies leveraging our model's probabilistic nature (NIAQUE-Ensemble).

\noindent\textbf{Results and Analysis:} 
The private leaderboard competition results along with their ranks are presented in Table~\ref{tab:kaggle-competition}, highlighting NIAQUE's capabilities against a wide range of well-established representative approaches. The effectiveness of transfer learning is evident in the performance progression: NIAQUE-Scratch (RMSLE: 0.15047), trained only on competition data, performs similarly to traditional baselines like XGBoost (0.15019). NIAQUE-Pretrain-100 (0.14823), pretrained on TabRegSet-101 excluding the Abalone dataset, shows significant improvement, demonstrating that knowledge from unrelated regression tasks can enhance performance. NIAQUE-Pretrain-full (0.14808), leveraging the complete TabRegSet-101, further improves performance by incorporating original Abalone information in the training mix and matching CatBoost (0.14783) without any feature engineering. The addition of automated feature engineering (OpenFE) into the competition dataset further improves NIAQUE's performance to 0.14556, significantly outperforming all baseline models and approaching the competition's winning score. Our best result comes from NIAQUE-Ensemble (RMSLE: 0.14423), which leverages the probabilistic nature of our model by combining predictions from different quantiles and model variants (Scratch and Pretrain-full). The ensemble benefits from quantile predictions (in addition to the median) that help correct for data imbalance, predicting higher values where models might underestimate and vice versa, achieving the 8th position on the private leaderboard, remarkably close to the winning score of 0.14374. 
This result is attained without extensive manual intervention---eschewing hand-engineered features in favor of transfer learning, standard ensemble techniques and automated feature extraction. Moreover, the model's probabilistic formulation is leveraged to further enhance point prediction accuracy, consistent with Bayesian estimation theory, where optimal estimators are typically derived as functions of the posterior distribution (e.g., the variance-optimal estimator is the posterior mean).


These results are particularly significant given that neural networks were generally considered ineffective for this task. Multiple participants reported neural approaches failing to achieve scores better than 0.15, with the competition's winner noting: \emph{``I tried different neural network architectures only to observe that none of them is competitive''}~\cite{first-place-abalone}, echoed by other participants: \emph{``yes, I also tried different neural network architectures only to observe that they could not reach beyond .15xx''}~\cite{first-place-abalone}. This real-world validation highlights NIAQUE's competitiveness against heavily engineered solutions and underscores the efficacy of integrating transfer learning with probabilistic modeling. Whereas the winning solution relied on an ensemble of 49 models with extensive manual tuning, our approach achieves comparable performance through principled transfer learning and uncertainty quantification, maintaining interpretability and necessitating minimal task-specific modifications.

\begin{table}[t]
\centering
\caption{Performance comparison on Flood Prediction competition. Higher R2-scores are better. Baseline results are from competition forums~\cite{playground-series-s4e5}. Rank represents the position on private leaderboard.}
\begin{tabular}{lcc}
\toprule
Model & R2 Score & Rank \\ 
\midrule
XGBoost~\cite{xgboost-tabnet-abalone, xgboost-floodpredict} & 0.842 & 2304 \\ 
LightGBM~\cite{lightgbm-abalone, lightgbm-floodpredict} & 0.766 & 2557 \\ 
CatBoost~\cite{catboost-abalone, catboost-floodpredict} & 0.845 & 1700 \\ 
TabNet~\cite{xgboost-tabnet-abalone} & 0.842 & 2304 \\ 
TabDPT~\cite{} & 0.804 & 2529 \\ 
TabPFN~\cite{} & 0.431 & - \\ 
NIAQUE-Scratch & 0.865 & 1099 \\ 
NIAQUE-Pretrain & 0.867 & 935 \\ 
Winner~\citep{first-place-abalone,first-place-flood-prediction} & 0.869 & 1 \\ 
\bottomrule
\end{tabular}
\label{tab:kaggle-flood}
\end{table}

\subsubsection{FloodPrediction Competition} \label{sec:kaggle-competitions-flood}

The flood prediction competition~\cite{playground-series-s4e5} presents a challenging regression task aimed at predicting flood event probabilities based on environmental features including MonsoonIntensity, TopographyDrainage, RiverManagement, and Deforestation. With 2,932 participants, this competition addresses a critical real-world problem where labeled data consists of flood probabilities rather than binary outcomes, making regression models more suitable than classification approaches.

\noindent\textbf{Methodology:} We adapted NIAQUE for this flood probability prediction task while maintaining its core probabilistic framework. The competition dataset comprises 1.12M training samples and 745K test samples, each containing relevant features derived from Flood Prediction Factors~\cite{flood-dataset}. Similar to our Abalone experiment, we evaluate transfer learning effectiveness through two variants: NIAQUE-Scratch (trained directly on competition data) and NIAQUE-Pretrain (pretrained on TabRegSet-101). The model is pretrained using our quantile loss framework and then fine-tuned with MSE Loss on the competition dataset. Notably, we found that automated feature engineering not only failed to improve performance but also significantly increased the feature space dimensionality, making several baseline methods (TabDPT and TabPFN) computationally intractable. Therefore, we focus on comparing the fundamental algorithmic capabilities without feature engineering enhancements.

\noindent\textbf{Results and Analysis:} The competition results (Table~\ref{tab:kaggle-flood}) demonstrate NIAQUE's strong performance in flood prediction. NIAQUE-Scratch achieves an R2-score of 0.865, significantly outperforming both traditional methods like XGBoost (0.842), CatBoost (0.845), and LightGBM (0.766), as well as newer deep learning approaches such as TabNet (0.842), TabDPT (0.804), and TabPFN (0.431). NIAQUE-Pretrain further improves performance to 0.867, approaching the winning score of 0.869, and securing rank 935 on the private leaderboard.

These results are particularly noteworthy given the scale and complexity of the dataset. Unlike the Abalone competition, where feature engineering played a crucial role, this competition highlights NIAQUE's ability to learn effective representations directly from raw features. The improvement from pretraining, though modest in absolute terms (0.002 increase in R2-score), represents a significant advancement in ranking (from 1099 to 935), demonstrating the value of transfer learning even in specialized environmental prediction tasks. The strong performance of our base model without extensive modifications or feature engineering suggests that NIAQUE's probabilistic framework and architecture are well-suited for large-scale regression tasks where the target variable represents underlying probabilities.

The relatively poor performance of some specialized tabular models (particularly TabPFN with R2-score of 0.431) on this large-scale dataset underscores the importance of scalability in real-world applications. NIAQUE maintains its computational efficiency while handling over a million samples, making it practical for deployment in real-world flood prediction systems where both accuracy and computational resources are critical considerations.


\subsection{Statistical Significance Analysis}\label{sec:accuracy-with-confidence-intervals}

To save space, we present benchmarking results with confidence intervals here. All confidence intervals are obtained by aggregating the evaluation results over 4 runs with different random seeds.

\begin{table}[hb]
    \centering
    \caption{Performance comparison across all metrics, with point prediction accuracy (SMAPE, AAD, RMSE, BIAS) and distributional accuracy (COV@95, CRPS). Lower values are better for all metrics except COV@95, where values closer to 95 are optimal. The results with 95\% confidence intervals derived from 4 random seed runs.}
    \renewcommand{\arraystretch}{1.2}
    \begin{tabular}{l|rrrrrr}
        \toprule
         & SMAPE & AAD & RMSE & BIAS & COV@95 & CRPS \\ 
        \midrule
        XGBoost-global & $31.4\pm 4.4$ & $0.574\pm 0.100$ & $1.056\pm 0.143$ & $-0.15\pm 0.05$ & $94.6\pm 0.3$ & $0.636\pm 0.165$ \\ 
        XGBoost-local & $25.6\pm 0.1$ & $0.433\pm 0.001$ & $0.883\pm 0.004$ & $-0.03\pm 0.01$ & $90.8\pm 0.2$ & $0.334\pm 0.001$ \\ 
        \midrule
        LightGBM-global & $27.5\pm 0.1$ & $0.475\pm 0.001$ & $0.930\pm 0.003$ & $-0.06\pm 0.01$ & $94.8\pm 0.1$ & $0.426\pm 0.017$ \\ 
        LightGBM-local & $25.7\pm 0.1$ & $0.427\pm 0.003$ & $0.865\pm 0.012$ & $-0.03\pm 0.01$ & $91.5\pm 0.2$ & $0.327\pm 0.001$ \\
        \midrule
        CATBOOST-global & $31.3\pm 0.2$ & $0.561\pm 0.006$ & $1.030\pm 0.009$ & $-0.12\pm 0.02$ & $94.9\pm 0.2$ & $0.443\pm 0.004$ \\ 
        CATBOOST-local & $24.3\pm 0.1$ & $0.408\pm 0.001$ & $0.840\pm 0.003$ & $-0.03\pm 0.01$ & $92.7\pm 0.1$ & $0.315\pm 0.001$ \\
        \midrule
        Transformer-global & $23.1\pm 0.3$ & $0.383\pm 0.008$ & $0.806\pm 0.015$ & $-0.01\pm 0.01$ & $94.6\pm 0.3$ & $0.272\pm 0.005$ \\ 
        NIAQUE-local & $22.8\pm 0.4$ & $0.377\pm 0.012$ & $0.797\pm 0.019$ & $-0.03\pm 0.01$ & $94.9\pm 0.4$ & $0.267\pm 0.011$ \\ 
        NIAQUE-global & $22.1\pm 0.1$ & $0.367\pm 0.002$ & $0.787\pm 0.005$ & $-0.02\pm 0.01$ & $94.6\pm 0.2$ & $0.261\pm 0.002$ \\ 
        \bottomrule
    \end{tabular}
    \label{tab:performance-full}
\end{table}

\clearpage

\section{Scalability Analysis of TabPFN and TabDPT} \label{appendix:tabpfn-tabdpt-baselines}

TabPFN and TabDPT represent important milestones in tabular deep learning, demonstrating competitive performance on small-scale datasets. However, our experiments reveal that both methods face significant scalability bottlenecks, limiting their utility in large-scale, real-world settings. Below, we outline these constraints in detail and motivate the need for scalable alternatives such as NIAQUE.

\subsection{Limitations of TabPFN}

Despite TabPFN’s impressive few-shot performance, its applicability is bounded by architectural and implementation constraints, as documented in the official repository~\cite{tabpfn-github}:
\begin{itemize}
    \item Maximum support for 10,000 training samples
    \item Limit of 500 features per instance
    \item Assumes all features are numerical, requiring preprocessing for categorical inputs
\end{itemize}

To adapt TabPFN for our use cases, we employed several workarounds:
\begin{itemize}
    \item For large datasets (e.g., the flood prediction dataset with 1.12M samples), we applied Random Forest-based subsampling, following official guidelines~\cite{tabpfn-github}.
    \item Categorical features were ordinally encoded to conform to TabPFN’s numerical input requirement.
\end{itemize}

These constraints, especially the aggressive downsampling, likely contributed to TabPFN’s suboptimal performance on large-scale tasks (e.g., R2 score of 0.431 on flood prediction).

\subsection{Scalability Challenges in TabDPT}

TabDPT, based on a transformer backbone, encounters scalability issues typical of attention-based models:
\begin{itemize}
    \item Memory consumption grows quadratically with context size due to self-attention
    \item Inference time scales poorly with dataset size
    \item Performance is sensitive to reductions in context size, making trade-offs between scalability and accuracy non-trivial
\end{itemize}

In our experiments:
\begin{itemize}
    \item We had to reduce the context window for the flood prediction dataset to fit within memory constraints. We reduce the context size from 8,192 in powers of 2. Context size of 1024 finally works.
    \item This reduction correlated with a decline in model performance (R2 score dropped to 0.804).
    \item The compute cost of running extensive hyperparameter tuning on large datasets proved impractical. TabDPT takes 12 hours of inference time for each hyperparameter configuration for flood prediction datatset on a V100 GPU.
\end{itemize}

These findings underscore the scalability limitations of current state-of-the-art tabular methods. While TabPFN and TabDPT perform well on smaller datasets, their architectures do not generalize efficiently to high-volume, high-dimensional data. In contrast, our proposed method, NIAQUE, is designed for scalability without sacrificing predictive accuracy. It maintains consistent performance across varying dataset sizes and avoids the need for extensive preprocessing, subsampling, or memory-intensive components. This makes it better suited for practical deployment in large-scale tabular learning settings.

\section{XGBoost Baseline} \label{sec:xgboost-baseline}


\begin{table}[t]
    \centering
    \caption{Ablation study of the XGBoost model.}
    \renewcommand{\arraystretch}{1.2}
    \begin{tabular}{rrr|rrrr|rr}
        \toprule
         type & \makecell{max \\ depth} & \makecell{learning \\rate} & $\smape$ & $\aad$  & $\bias$ & $\rmse$ & CRPS & \makecell{COVERAGE \\ @ 95} \\ 
        \midrule
        global & 8 & 0.02 & 31.4 & 0.574 & -0.15 & 1.056 & 0.636 & 94.6
        \\
        global & 16 & 0.02 & 25.7 & 0.441 & -0.07 & 0.864 & 0.484 & 91.5
        \\ 
        global & 32 & 0.02 & 24.1 & 0.402 & -0.05 & 0.800 & 0.353 & 80.0
        \\ 
        global & 40 & 0.02 & 24.6 & 0.414 & -0.05 & 0.815 & 0.378 & 78.2
        \\ 
        global & 48 & 0.02 & 24.1 & 0.397 & -0.04 & 0.785 & 0.362 & 74.8
        \\ 
        global & 96 & 0.02 & 23.8 & 0.384 & -0.03 & 0.769 & 0.346 & 64.9
        \\
        local & 16 & 0.02 & 23.0 & 0.367 & -0.00 & 0.753 & 0.317 & 52.0 \\ 
        local & 12 & 0.02 & 22.7 & 0.369 & -0.01 & 0.756 & 0.304 & 66.0
        \\
        local & 8 & 0.02 & 22.4 & 0.372 & -0.02 & 0.773 & 0.294 & 82.3 \\ 
        local & 8 & 0.05 & 22.5 & 0.373 & -0.02 & 0.773 & 0.291 & 82.4 \\ 
        local & 6 & 0.02 & 22.7 & 0.382 & -0.02 & 0.795 & 0.298 & 87.3
        \\
        local & 4 & 0.02 & 24.1 & 0.412 & -0.03 & 0.847 & 0.318 & 90.2
        \\
        local & 3 & 0.02 & 25.6 & 0.433 & -0.03 & 0.883 & 0.334 & 90.8 
        \\
        \bottomrule
    \end{tabular}
    \label{tab:xgboost_ablation}
\end{table}

        
        

\clearpage
\section{CATBoost Baseline} \label{sec:catboost-baseline}

The CATBoost is trained using the standard package via \verb|pip install catboost| using \verb|grow_policy = Depthwise|. The explored hyper-parqameter grid appears in~\Cref{tab:catboost_ablation}.

\Cref{tab:catboost_num_quantiles} shows CATBoost accuracy as a function of the number of quantiles. Quantiles are generated using linspace grid \verb|np.linspace(0.01, 0.99, num_quantiles)|. We recover the best overall result for the case of 3 quantiles, and increasing the number of quantiles leads to quickly deteriorating metrics. It appears that CATBoost is unfit to solve complex multi-quantile problems.

\begin{table}[t]
    \centering
    \caption{Ablation study of the CATBoost model.}
    \renewcommand{\arraystretch}{1.2}
    \begin{tabular}{rrr|rrrr|rr}
        \toprule
         type & depth & \makecell{min data \\ in leaf} & $\smape$ & $\aad$  & $\bias$ & $\rmse$ & CRPS & \makecell{COVERAGE \\ @ 95} \\ 
        \midrule
        global & 16 & 50 &  31.4 & 0.565 & -0.12 & 1.036 & 0.442 & 94.2
        \\ 
        global & 16 & 100 & 31.3 & 0.561 & -0.12 & 1.030 & 0.443 & 94.9
        \\ 
        global & 16 & 200 & 31.6 & 0.569 & -0.13 & 1.041 & 0.445 & 94.2
        \\
        global & 8 & 100 & 41.1 & 0.785 & -0.26 & 1.324 & 0.602 & 94.3
        \\ 
        
        local & 3 & 50 & 24.3 & 0.409 & -0.03 & 0.841 & 0.316 & 92.7
        \\ 
        local & 3 & 100 & 24.3 & 0.407 & -0.03 & 0.843 & 0.317 & 92.7
        \\ 
        local & 3 & 200 & 24.3 & 0.408 & -0.03 & 0.840 & 0.315 & 92.7
        \\ 

        local & 5 & 50  & 22.2 & 0.373 & -0.02 & 0.785 & 0.285 & 90.7
        \\ 
        local & 5 & 100  & 22.3 & 0.374 & -0.02 & 0.786 & 0.285 & 91.3
        \\ 
        local & 5 & 200  & 22.4 & 0.378 & -0.02 & 0.791 & 0.288 & 91.6
        \\ 
        
         local & 7 & 50  & 21.5 & 0.359 & -0.02 & 0.761 & 0.272 & 87.2
        \\
         local & 7 & 100  & 21.6 & 0.362 & -0.02 & 0.765 & 0.273 & 88.6
        \\
         local & 7 & 200  & 21.8 & 0.366 & -0.02 & 0.772 & 0.277 & 89.9
        \\

        \bottomrule
    \end{tabular}
    \label{tab:catboost_ablation}
\end{table}

\begin{table}[t]
    \centering
    \caption{CATBoost accuracy as a function of the number of quantiles.}
    \renewcommand{\arraystretch}{1.2}
    \begin{tabular}{cccc|rrrr|rr}
        \toprule
         type & depth & \makecell{min data \\ in leaf} & \makecell{num \\ quantiles} & $\smape$ & $\aad$  & $\bias$ & $\rmse$ & CRPS & \makecell{COVERAGE \\ @ 95} \\ 
        \midrule

        global & 16 & 100 & 3 & 31.3 & 0.561 & -0.12 & 1.030 & 0.443 & 94.9
        \\
        global & 16 & 100 & 5 & 35.0 & 0.665 & -0.13 & 1.183 & 0.482 & 96.2
        \\
        global & 16 & 100 & 7 & 38.5 & 0.746 & -0.18 & 1.265 & 0.533 & 96.2
        \\
        global & 16 & 100 & 9 & 43.7 & 0.879 & -0.25 & 1.437 & 0.622 & 96.2
        \\
        global & 16 & 100 & 51 &  68.9 & 1.538 & -0.53 & 2.132 & 1.036 & 95.5
        \\
        \midrule
        local & 7 & 100 & 3  &  21.5 & 0.359 & -0.02 & 0.761 & 0.272 & 87.2
        \\
        local & 7 & 100 & 9 &  23.9 & 0.399 & -0.03 & 0.823 & 0.284 & 92.4
        \\
        local & 7 & 100 & 51 & 30.3 & 0.525 & -0.09 & 1.079 & 0.369 & 92.1
        \\
        local & 16 & 100 & 51 & 30.2 & 0.514 & -0.09 & 1.055 & 0.362 & 92.4
        \\
        \bottomrule
    \end{tabular}
    \label{tab:catboost_num_quantiles}
\end{table}

\clearpage
\section{LightGBM Baseline} \label{sec:lightgbm-baseline}

\begin{table}[t]
    \centering
    \caption{Ablation study of the LightGBM model.}
    \renewcommand{\arraystretch}{1.2}
    \begin{tabular}{rrrr|rrrr|rr}
        \toprule
         type & max\_depth & \makecell{num \\ leaves} & \makecell{learning \\ rate} & $\smape$ &  $\aad$  & $\bias$ & $\rmse$ & CRPS & \makecell{COVERAGE \\ @ 95} \\ 
        \midrule
        global & -1 & 10 & 0.05  & 35.6 & 0.661 & -0.17 & 1.199 & 0.804 & 95.2
        \\ 
        global & -1 & 20 & 0.05 & 30.9 & 0.554 & -0.11 & 1.034 & 0.566 & 95.4
        \\ 
        global & -1 & 40 & 0.05 & 27.5 & 0.475 & -0.06 & 0.930 & 0.426 & 94.8
        \\ 
         global & -1 & 100 & 0.05 & 24.6 & 0.417 & -0.03 & 0.852 & 0.342 & 93.3
        \\ 
        global & -1 & 200 & 0.05  & 23.4 & 0.393 & -0.02 & 0.813 & 0.32 & 92.3
        \\ 
        global & -1 & 400 & 0.05 & 23.6 & 0.379 & -0.02 & 0.786 & 0.305 & 90.9
        \\
        global & 3 & 10 & 0.05  & 50.7 & 1.084 & -0.49 & 1.763 & 1.013 & 94.1
        \\ 
        global & 3 & 20 & 0.05 & 50.7 & 1.084 & -0.49 & 1.763 & 1.013 & 94.1
        \\ 
        global & 3 & 40 & 0.05 & 50.7 & 1.084 & -0.49 & 1.763 & 1.013 & 94.1
        \\ 
         global & 3 & 100 & 0.05 & 50.7 & 1.084 & -0.49 & 1.763 & 1.013 & 94.1
        \\ 
        global & 3 & 200 & 0.05  & 50.7 & 1.084 & -0.49 & 1.763 & 1.013 & 94.1
        \\ 
        global & 3 & 400 & 0.05 & 50.7 & 1.084 & -0.49 & 1.763 & 1.013 & 94.1
        \\
        global & 5 & 10 & 0.05  & 39.1 & 0.768 & -0.25 & 1.341 & 0.856 & 94.8
        \\ 
        global & 5 & 20 & 0.05 & 39.0 & 0.76 & -0.26 & 1.327 & 0.863 & 94.8
        \\ 
        global & 5 & 40 & 0.05 & 39.0 & 0.759 & -0.26 & 1.328 & 0.864 & 94.8
        \\ 
         global & 5 & 100 & 0.05 & 39.0 & 0.759 & -0.26 & 1.328 & 0.864 & 94.8
        \\ 
        global & 5 & 200 & 0.05  & 39.0 & 0.759 & -0.26 & 1.328 & 0.864 & 94.8
        \\ 
        global & 5 & 400 & 0.05 & 39.0 & 0.759 & -0.26 & 1.328 & 0.864 & 94.8
        \\
        global & 10 & 10 & 0.05  & 35.6 & 0.661 & -0.17 & 1.199 & 0.804 & 95.2
        \\ 
        global & 10 & 20 & 0.05 & 31.5 & 0.572 & -0.14 & 1.054 & 0.59 & 95.4
        \\ 
        global & 10 & 40 & 0.05 & 29.8 & 0.537 & -0.13 & 1.001 & 0.575 & 95.2
        \\ 
         global & 10 & 100 & 0.05 & 29.5 & 0.528 & -0.12 & 0.991 & 0.577 & 95.2
        \\ 
        global & 10 & 200 & 0.05  & 29.2 & 0.522 & -0.12 & 0.981 & 0.576 & 95.0
        \\ 
        global & 10 & 400 & 0.05 & 29.1 & 0.52 & -0.12 & 0.975 & 0.582 & 95.1
        \\
        global & 20 & 10 & 0.05  & 35.6 & 0.661 & -0.17 & 1.199 & 0.804 & 95.2
        \\ 
        global & 20 & 20 & 0.05 & 30.9 & 0.554 & -0.11 & 1.034 & 0.566 & 95.4
        \\ 
        global & 20 & 40 & 0.05 & 27.1 & 0.468 & -0.07 & 0.913 & 0.512 & 95.2
        \\ 
         global & 20 & 100 & 0.05 & 25.5 & 0.435 & -0.06 & 0.864 & 0.496 & 94.9
        \\ 
        global & 20 & 200 & 0.05  & 25.0 & 0.424 & -0.06 & 0.846 & 0.488 & 94.3
        \\ 
        global & 20 & 400 & 0.05 & 24.3 & 0.41 & -0.05 & 0.823 & 0.482 & 93.6
        \\
        global & 40 & 10 & 0.05  & 35.6 & 0.661 & -0.17 & 1.199 & 0.804 & 95.2
        \\ 
        global & 40 & 20 & 0.05 & 30.9 & 0.554 & -0.11 & 1.034 & 0.566 & 95.4
        \\ 
        global & 40 & 40 & 0.05 & 27.8 & 0.481 & -0.05 & 0.913 & 0.431 & 94.7
        \\ 
         global & 40 & 100 & 0.05 & 24.7 & 0.419 & -0.04 & 0.848 & 0.348 & 93.5
        \\ 
        global & 40 & 200 & 0.05  & 23.5 & 0.395 & -0.03 & 0.811 & 0.332 & 92.7
        \\ 
        global & 40 & 400 & 0.05 & 23.2 & 0.383 & -0.03 & 0.791 & 0.322 & 92.0
        \\
        \\
        \bottomrule
    \end{tabular}
    \label{tab:lightgbm_ablation}
\end{table}

\begin{table}[t]
    \centering
    \caption{Ablation study of the LightGBM model.}
    \renewcommand{\arraystretch}{1.2}
    \begin{tabular}{rrrr|rrrr|rr}
        \toprule
         type & max\_depth & \makecell{num \\ leaves} & \makecell{learning \\ rate} & $\smape$ &  $\aad$  & $\bias$ & $\rmse$ & CRPS & \makecell{COVERAGE \\ @ 95} \\ 
        \midrule

        local & -1 & 5 & 0.05 & 23.8 & 0.399 & -0.03 & 0.823 & 0.319 & 90.6
        \\
        local & -1 & 10 & 0.05 & 22.5 & 0.376 & -0.02 & 0.786 & 0.301 & 88.9
        \\
         local & -1 & 20 & 0.05 & 21.9 & 0.364 & -0.02 & 0.766 & 0.289 & 86.5
         \\
         local & -1 & 50 & 0.05 & 21.6 & 0.355 & -0.01 & 0.752 & 0.278 & 82.6
        \\
        local & 2 & 5 & 0.05 & 25.7 & 0.427 & -0.03 & 0.865 & 0.327 & 91.5
        \\
        local & 2 & 10 & 0.05 & 25.7 & 0.427 & -0.03 & 0.865 & 0.327 & 91.5
        \\
         local & 2 & 20 & 0.05 & 25.7 & 0.427 & -0.03 & 0.865 & 0.327 & 91.5
         \\
         local & 2 & 50 & 0.05 & 25.7 & 0.427 & -0.03 & 0.865 & 0.327 & 91.5
        \\
        local & 3 & 5 & 0.05 & 24.3 & 0.404 & -0.03 & 0.83 & 0.318 & 90.7
        \\
        local & 3 & 10 & 0.05 & 23.9 & 0.396 & -0.03 & 0.818 & 0.304 & 90.4
        \\ 
         local & 3 & 20 & 0.05 & 23.9 & 0.396 & -0.03 & 0.818 & 0.304 & 90.4
         \\ 
         local & 3 & 50 & 0.05 & 23.9 & 0.396 & -0.03 & 0.818 & 0.304 & 90.4
        \\ 
        local & 5 & 5 & 0.05 & 23.8 & 0.399 & -0.03 & 0.823 & 0.319 & 90.6
        \\
        local & 5 & 10 & 0.05 & 22.7 & 0.379 & -0.02 & 0.79 & 0.3 & 89.1
        \\ 
         local & 5 & 20 & 0.05 & 22.3 & 0.37 & -0.02 & 0.776 & 0.287 & 87.6
         \\ 
         local & 5 & 50 & 0.05 & 22.2 & 0.368 & -0.02 & 0.773 & 0.285 & 87.4
        \\ 
        \bottomrule
    \end{tabular}
    \label{tab:lightgbm_ablation_local}
\end{table}

\clearpage
\section{Transformer Baseline} \label{sec:transformer-baseline}

\begin{figure}[t]
    \centering
    \includegraphics[width=0.8\linewidth]{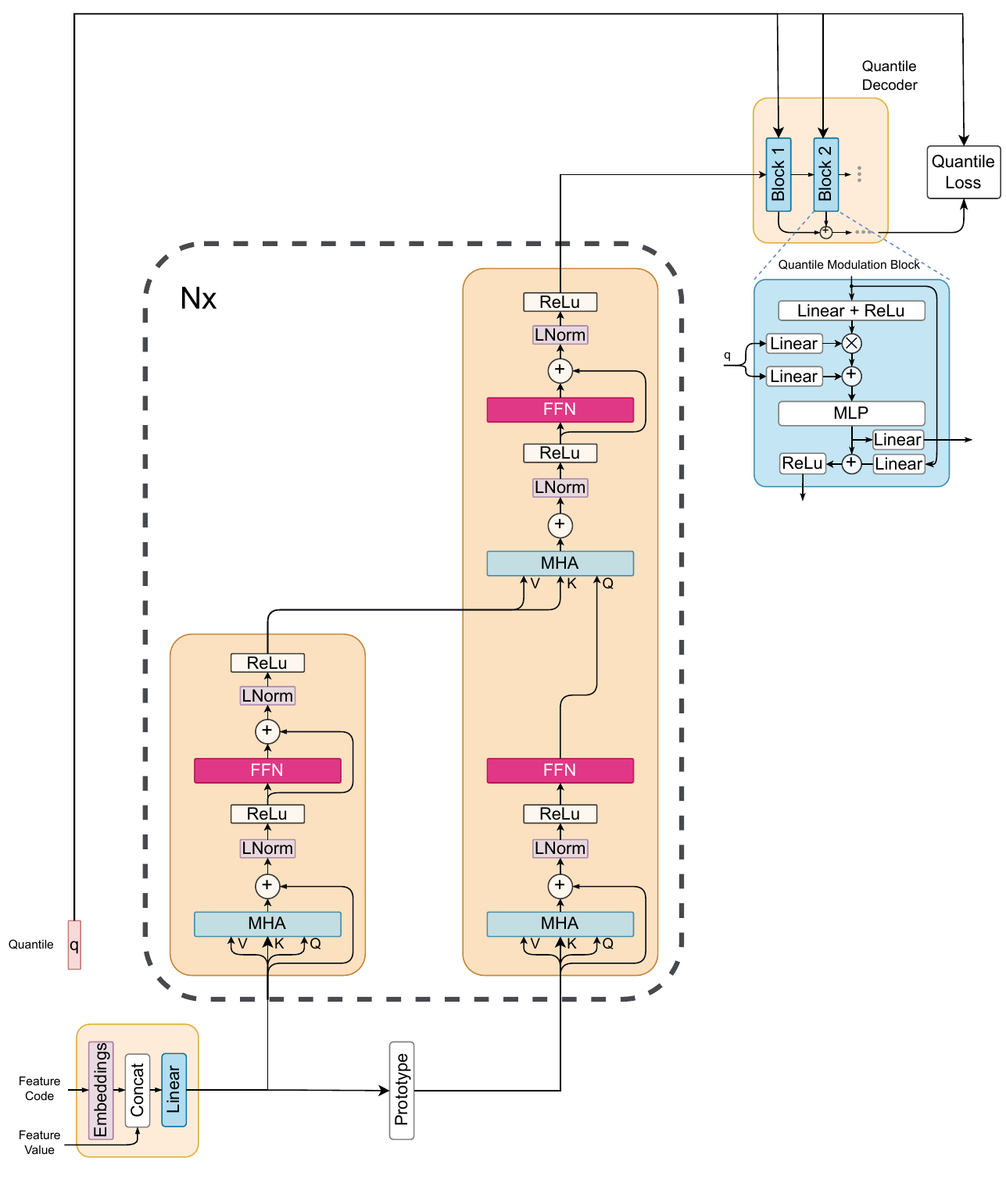}
    \setlength{\belowcaptionskip}{-10pt}
    \caption{Transformer baseline used in our experiments. The feature encoding module is replaced with transformer block. Feature encoding is implemented via self-attention. The extraction of feature encoding is done by applying cross-attention between the prototype of input features and the output of self-attention. This operation is repeated several times corresponding to the number of blocks in transformer encoder.}
    \label{fig:transformer-baseline}
\end{figure}

The ablation study of the transformer architecture is presented in \Cref{tab:transformer-ablation}. It shows that in general, increasing the number of transformer blocks improves accuracy, however, at 8-10 blocks we clearly see diminishing returns. Dropout helps to gain better empirical coverage of the 95\% confidence interval, but this happens at the expense of point prediction accuracy. Finally, the decoder query that is used to produce the feature embedding that is fed to the quantile decoder can be implemented in two principled ways. First, the scheme depicted in \Cref{fig:transformer-baseline}, uses the prototype of features supplied to the encoder. We call it the prototype scheme. Second, the prototype can be replaced by a learnable embedding. Comparing the last and third rows in \Cref{tab:transformer-ablation}, we conclude that the prototype scheme is a clear winner.


\begin{table}[t]
    \centering
    \caption{Ablation study of the Transformer architecture.}
    \renewcommand{\arraystretch}{1.2}
    \begin{tabular}{lllll|rrrr|rr}
        \toprule
         query & d\_model & width & blocks & dp & $\smape$ & $\aad$  & $\bias$ & $\rmse$ & CRPS & \makecell{COVERAGE \\ @ 95} \\ 
        
        \midrule
         proto & 256 & 256 & 4 & 0.1 & 25.6 & 0.462 & -0.01 & 0.918 & 0.313 & 95.2 \\ 
         proto & 256 & 1024 & 4 & 0.1 & 24.5 & 0.414 & -0.02 & 0.845 & 0.292 & 95.1 \\ 
         \midrule
         proto & 256 & 256 & 6 & 0.1 & 23.7 & 0.397 & -0.01 & 0.824 & 0.281 & 94.9 \\ 
         proto & 256 & 512 & 6 & 0.2 &  \\ 
         proto & 256 & 1024 & 6 & 0.1 & 24.3 & 0.407 & -0.01 & 0.840 & 0.287 & 94.9 \\
         proto & 256 & 1024 & 6 & 0.0 & 26.5 & 0.477 & -0.04 & 0.980 & 0.334 & 93.0 \\
         
         \midrule
         proto & 256 & 512 & 8 & 0.0 & 23.3 & 0.388 & -0.03 & 0.814 & 0.276 & 94.3
         \\ 
         proto & 256 & 1024 & 8 & 0.0 & 23.1 & 0.383 & -0.02 & 0.806 & 0.272 & 94.6
         \\ 
         proto & 256 & 1024 & 8 & 0.1 & 23.1 & 0.384 & -0.01 & 0.809 & 0.272 & 94.6 \\ 
         proto & 256 & 512 & 10 & 0.0 & 23.0 & 0.384 & -0.03 & 0.814 & 0.273 & 94.2 \\ 
         proto & 256 & 1024 & 10 & 0.1 & 24.3 & 0.407 & -0.01 & 0.840 & 0.287 & 94.9 \\ 
         \midrule
         proto & 512 & 1024 & 6 & 0.1 &  \\ 
         \midrule
         learn & 256 & 256 & 6 & 0.2 & 35.0 & 0.722 &  -0.16 & 1.406 & 0.489 & 93.9 
         \\ 
        \bottomrule
    \end{tabular}
    \label{tab:transformer-ablation}
\end{table}

\clearpage
\section{\name{}-Local Baseline} \label{sec:niaque-local-baseline}

NIAQUE-local baseline is trained on each dataset individually using the same overall training framework as discussed in the main manuscript for the NIAQUE-global, with the following exceptions. The number of training epochs for each dataset is fixed at 1200, the batch size is set to 256, feature dropout is disabled. Finally, for each dataset we select the best model to be evaluated by monitoring the loss on validation set every epoch.

\begin{table}[t]
    \centering
    \caption{Ablation study of NIAQUE-local model.}
    \renewcommand{\arraystretch}{1.2}
    \begin{tabular}{llll|rrrr|rr}
        \toprule
         blocks & width & dp & layers & $\smape$ & $\aad$  & $\bias$ & $\rmse$ & CRPS & \makecell{COVERAGE \\ @ 95} \\ 
        \midrule
         
        2 & 64 & 0.0 & 3 & 24.2 & 0.414 & -0.03 & 0.848 & 0.292 & 95.1 \\ 

        2 & 128 & 0.0 & 3 & 22.8 & 0.381 & -0.02 & 0.804 & 0.270 & 94.5 \\ 

        2 & 256 & 0.0 & 3 & 22.1 & 0.365 & -0.02 & 0.786 & 0.260 & 94.0 \\

        2 & 512 & 0.0 & 3 & 21.9 & 0.360 & -0.02 & 0.781 & 0.257 & 92.7 \\

        \midrule

        2 & 64 & 0.1 & 3 & 24.7 & 0.431 & -0.07 & 0.855 & 0.305 & 93.3 \\

        2 & 128 & 0.1 & 3 &  23.1 & 0.389 & -0.04 & 0.81 & 0.276 & 94.0 \\

        2 & 256 & 0.1 & 3 & 22.2 & 0.369 & -0.02 & 0.79 & 0.263 & 94.0 \\ 

        2 & 512 & 0.1 & 3 &  22.0 & 0.361 & -0.02 & 0.779 & 0.257 & 93.5 \\

        \midrule

        2 & 64 & 0.0 & 2 & 24.5 & 0.419 & -0.03 & 0.852 & 0.296 & 95.0 \\

        2 & 128 & 0.0 & 2 & 23.4 & 0.391 & -0.02 & 0.815 & 0.276 & 94.7 \\

        2 & 256 & 0.0 & 2 &  22.3 & 0.368 & -0.02 & 0.783 & 0.262 & 94.1 \\

        2 & 512 & 0.0 & 2 & 22.1 & 0.363 & -0.03 & 0.780 & 0.259 & 92.9 \\

        \midrule

        4 & 64 & 0.0 & 2 & 23.8 & 0.399 & -0.02 & 0.828 & 0.282 & 95.1 \\

        4 & 128 & 0.0 & 2 & 22.8 & 0.377 & -0.03 & 0.797 & 0.267 & 94.9 \\
        
        4 & 256 & 0.0 & 2 & 22.0 & 0.363 & -0.02 & 0.788 & 0.259 & 93.5 \\

        4 & 512 & 0.0 & 2 & 22.0 & 0.359 & -0.02 & 0.785 & 0.257 & 92.0 \\

        \midrule

        4 & 64 & 0.1 & 2 & 23.8 & 0.401 & -0.03 & 0.829 & 0.284 & 94.3 \\

        4 & 128 & 0.1 & 2 & 22.9 & 0.379 & -0.03 & 0.801 & 0.267 & 94.6 \\

        4 & 256 & 0.1 & 2 & 22.1 & 0.363 & -0.03 & 0.786 & 0.259 & 93.5 \\

        4 & 512 & 0.1 & 2 & 22.0 & 0.360 & -0.03 & 0.781 & 0.257 & 92.4 \\

        \midrule
        
        8 & 128 & 0.0 & 2 & 23.0 & 0.381 & -0.02 & 0.798 & 0.27 & 95.7 \\

        \bottomrule
    \end{tabular}
    \label{tab:niaque_local_ablation}
\end{table}

\clearpage
\section{\name{} Training Details and Ablation Studies}
\label{sec:niaque-ablation-appendix}

To train both NIAQUE and Transformer models we use feature dropout defined as follows. Given dropout probability $\textrm{dp}$, we toss a coin with probability $\sqrt{\textrm{dp}}$ to determine if the dropout event is going to happen at all for a given batch. If this happens, we remove each feature from the batch, again with probability $\sqrt{\textrm{dp}}$. This way each feature has probability $\textrm{dp}$ of being removed from a given batch and there is a probability $\sqrt{\textrm{dp}}$ that the model will see all features intact in a given batch. The intuition behind this design is that we want to expose the model to all features most of the time, but we also want to create many situations with some feature combinations missing.


\textbf{Input log transformation} defined in~\cref{eqn:log_transformation} is important to ensure the success of the training, as follows both from Table~\ref{tab:lightgbm_ablation} and~\Cref{fig:with_and_without_log}. The introduction of log-transform makes learning curves well-behaved and smooth and translates into much better accuracy.

\textbf{Adding samples containing only one of the features} as input does not significantly affect accuracy. At the same time, the addition of single-feature training rows has very strong effect on the effectiveness of \name{}'s interpretability mechanism. When rows with single feature input are added (\Cref{fig:niaque_interpretability_ablation:0.05,fig:niaque_interpretability_ablation:0.1}), \name{} demonstrates very clear accuracy degradation when top features are removed and insignificant degradation when bottom features are removed. When rows with single feature input are \emph{not} added (\Cref{fig:niaque_interpretability_ablation:0.0}), the discrimination between strong and weak features is poor, with removal of top and bottom features having approximately the same effect across datasets.

\begin{table}[t]
    \centering
    \caption{Ablation study of NIAQUE model.}
    \renewcommand{\arraystretch}{1.2}
    \begin{tabular}{llllll|rrrr|rr}
        \toprule
         blocks & width & dp & layers & singles & \makecell{log \\ input} & $\smape$ & $\aad$  & $\bias$ & $\rmse$ & CRPS & \makecell{COVERAGE \\ @ 95} \\ 
        \midrule

        1 & 1024 & 0.2 & 2 & 5\% & yes & 25.6 & 0.433 & -0.04 & 0.864 & 0.306 & 96.5 \\

        \midrule
        2 & 1024 & 0.2 & 2 & 5\% & yes & 23.1 & 0.384 & -0.02 & 0.802 & 0.272 & 95.7 \\

        2 & 1024 & 0.2 & 3 & 5\% & yes & 22.7 & 0.377 & -0.03 & 0.796 & 0.267 & 95.6 \\

        \midrule
        4 & 1024 & 0.2 & 2 & 5\% & yes  & 22.1 & 0.367 & -0.02 & 0.787 & 0.261 & 94.6 \\ 
        4 & 1024 & 0.2 & 3 & 5\% & yes  & 22.1 & 0.367 & -0.02 & 0.792 & 0.262 & 94.6 \\ 
        8 & 1024 & 0.2 & 2 & 5\% & yes  & 22.0 & 0.366 & -0.02 & 0.798 & 0.264 & 92.7 \\

        \midrule
        4 & 512 & 0.2 & 2 & 0\% & yes & 22.5 & 0.372 & -0.02 & 0.791 & 0.264 & 95.4 \\ 

        4 & 1024 & 0.2 & 2 & 0\% & yes  & 22.1 & 0.366 & -0.02 & 0.791 & 0.261 & 94.2 \\ 

        4 & 1024 & 0.3 & 2 & 0\% & yes  & 22.1 & 0.367 & -0.02 & 0.787 & 0.260 & 94.7 \\ 
        4 & 1024 & 0.4 & 2 & 0\% & yes  & 22.2 & 0.370 & -0.02 & 0.791 & 0.263 & 95.1 \\ 
        4 & 2048 & 0.3 & 2 & 0\% & yes  & 22.1 & 0.366 & -0.02 & 0.795 & 0.263 & 93.4 \\ 
        
        \midrule

        4 & 1024 & 0.2 & 2 & 5\% & no  & 31.4 & 0.530 & -0.066 & 1.017 & 0.371 & 95.6 \\ 
        
        \bottomrule
    \end{tabular}
    \label{tab:lightgbm_ablation}
\end{table}

\begin{figure*}[t]
    \centering
    \includegraphics[width=\textwidth]{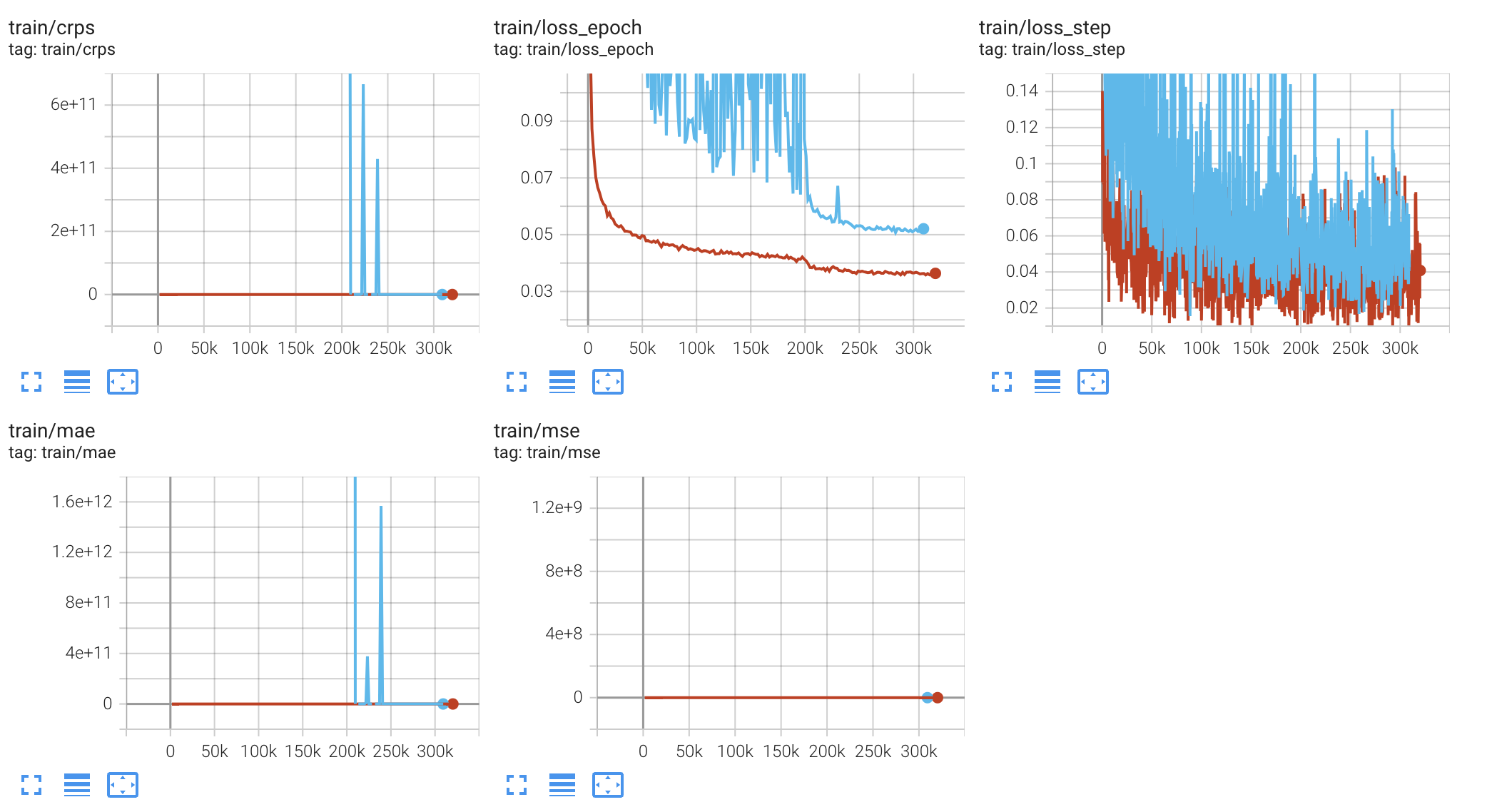}
    \caption{Training losses with (dark red) and without (blue) input value log-transform~\cref{eqn:log_transformation}. The introduction of log-transform makes learning curves well-behaved and smooth.}
    \label{fig:with_and_without_log}
\end{figure*}

\begin{figure*}
    \centering
    
    \begin{subfigure}[b]{0.3\textwidth}
        \centering
        \includegraphics[width=\textwidth]{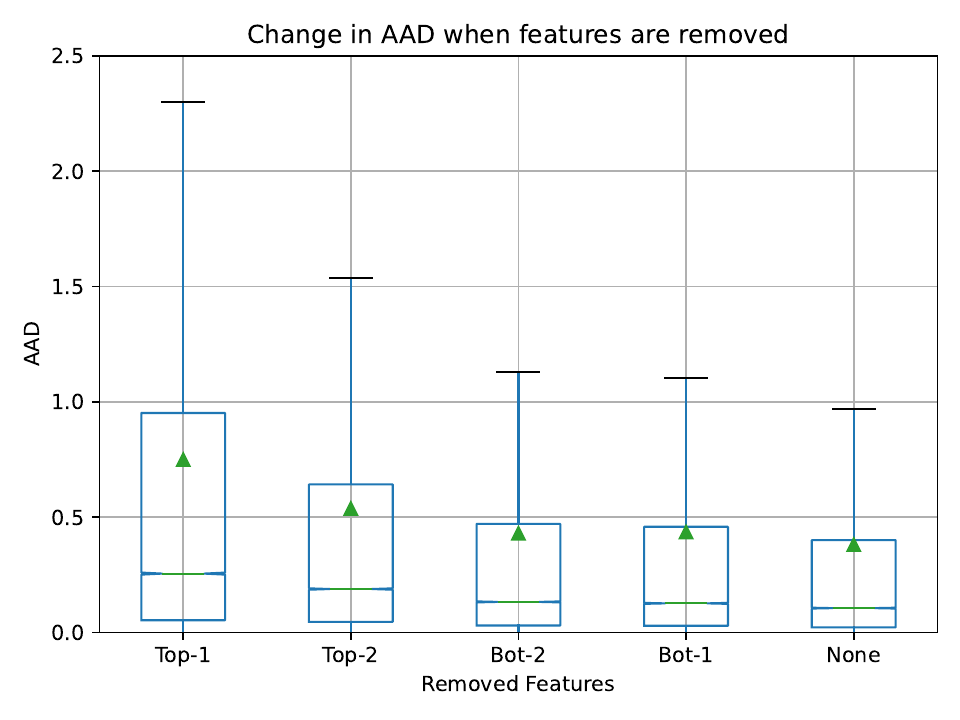}
        \caption{Single features, 10\%}
        \label{fig:niaque_interpretability_ablation:0.1}
    \end{subfigure}
    \begin{subfigure}[b]{0.3\textwidth}
        \centering
        \includegraphics[width=\textwidth]{Figures/interpretability_analysis_singles0.05.pdf}
        \caption{Single features, 5\%}
        \label{fig:niaque_interpretability_ablation:0.05}
    \end{subfigure}
    \begin{subfigure}[b]{0.3\textwidth}
        \centering
        \includegraphics[width=\textwidth]{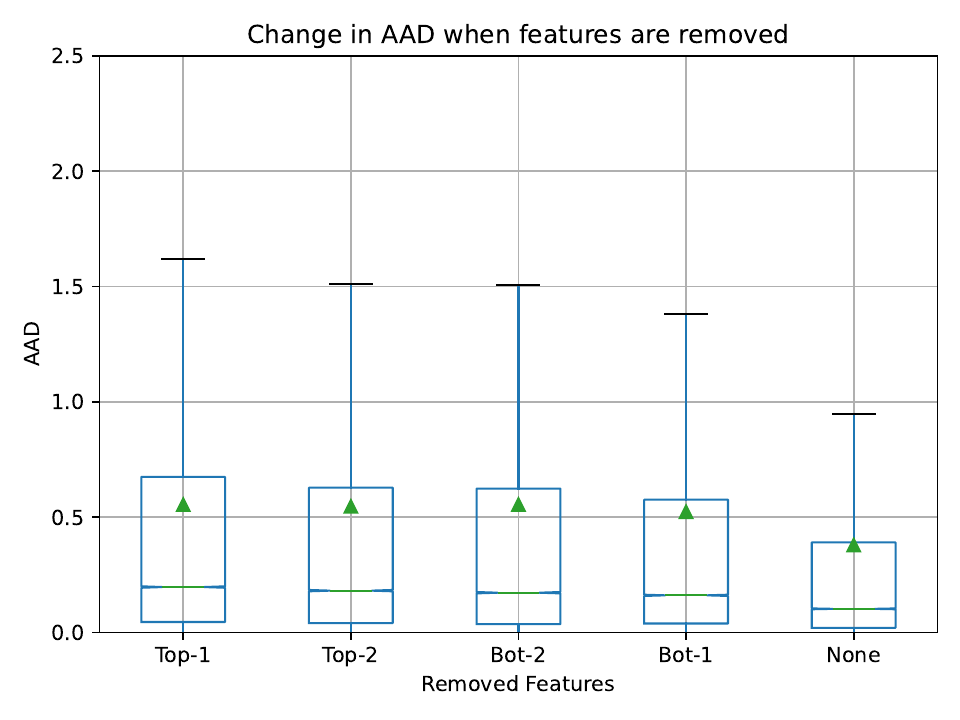}
        \caption{Single features, 0\%}
        \label{fig:niaque_interpretability_ablation:0.0}
    \end{subfigure}
    \caption{The effect of adding training rows containing only one of the input features as \name{} input. When rows with single feature input are added (\Cref{fig:niaque_interpretability_ablation:0.05,fig:niaque_interpretability_ablation:0.1}), \name{} demonstrates very clear accuracy degradation when top features are removed and insignificant degradation when bottom features are removed. When rows with single feature input are \emph{not} added (\Cref{fig:niaque_interpretability_ablation:0.0}), the discrimination between strong and weak features is poor, with removal of top and bottom features having approximately the same effect across datasets.}
    \label{fig:niaque_interpretability_ablation}
\end{figure*}

\end{document}